\documentclass[review,1p]{elsarticle}

\usepackage{microtype}
\usepackage{graphicx}
\usepackage{subfigure}
\usepackage{booktabs} 
\usepackage{adjustbox}
\usepackage{hyperref}

\usepackage{amsmath}
\usepackage{amssymb}
\usepackage{mathtools}
\usepackage{amsthm}

\newcommand{\A}[3]{\alpha_{#1 \leftarrow #2}(#3)}
\newcommand{\R}[3]{\rho_{#1 \rightarrow #2}(#3)}
\DeclareMathOperator*{\argmax}{arg\,max}
\DeclareMathOperator*{\argmin}{arg\,min}
\usepackage{natbib}  
\usepackage{caption} 
\usepackage{algorithm}
\usepackage{algorithmic}

\usepackage[capitalize,noabbrev]{cleveref}

\theoremstyle{plain}
\newtheorem{theorem}{Theorem}[section]

\theoremstyle{definition}

\theoremstyle{remark}

\usepackage[textsize=tiny]{todonotes}

\begin{document}
\title{Compositional Clustering: Applications to Multi-Label  Object Recognition and Speaker Identification}
\author[1]{Zeqian Li}
\ead{zli14@wpi.edu}
\author[1]{Xinlu He}
\ead{xhe4@wpi.edu}
\author[1]{Jacob Whitehill\corref{cor1}}
\ead{jrwhitehill@wpi.edu}
\affiliation[1]{organization={Worcester Polytechnic Institute}, addressline={100 Institute Road},
postcode={01609}, city={Worcester}, country={USA}}
\cortext[cor1]{Corresponding author}


\begin{abstract}
We consider a novel clustering task in which 
clusters can have compositional relationships, e.g.,
one cluster contains images of rectangles, one  contains images of circles, and a third (compositional) cluster contains images with both objects. In contrast to hierarchical clustering in which a parent cluster represents the \emph{intersection} of properties of the child clusters, our problem is about finding compositional clusters that represent the \emph{union} of the properties of the constituent clusters.
This task is motivated by
recently developed few-shot learning and embedding models
\cite{alfassy2019laso,li2021compositional} can distinguish
the label \emph{sets}, not just the individual labels, assigned to the examples.
We propose three new algorithms -- Compositional Affinity Propagation (CAP), Compositional $k$-means (CKM), and Greedy Compositional Reassignment (GCR) -- that can  partition examples into coherent groups and infer the compositional structure among them. We show promising results, compared to popular algorithms such as Gaussian mixtures, Fuzzy $c$-means, and Agglomerative Clustering, on the
OmniGlot and LibriSpeech datasets. Our work has applications to open-world multi-label object recognition
and speaker identification \& diarization with simultaneous speech from multiple speakers.
\end{abstract}

\maketitle

\section{Introduction}
We consider a new kind of clustering problem in which clusters have compositional structure, in the
sense that each example in one cluster may exhibit the \emph{union} of the properties found in another set of clusters.
The goal is not just to partition the data into distinct and coherent groups, but also to infer the
compositional relationships among the groups.
This scenario arises in speaker diarization (i.e., infer who is speaking when from an audio wave)
in the presence of simultaneous speech from multiple speakers \cite{bullock2020overlap,zelenak2012simultaneous}, which occurs frequently in real-world speech settings: 
The audio at each time $t$ is generated as a \emph{composition} of the voices of all the people 
speaking at time $t$, and the goal is to cluster the audio samples, over all timesteps, into sets of 
speakers. Hence, if there are 2 people who sometimes speak by themselves and sometimes speak simultaneously, then
the clusters would correspond to the speaker sets $\{1\}$, $\{2\}$, and $\{1,2\}$ -- the third cluster is not a third independent speaker, but rather the composition of the first two speakers.
An analogous scenario arises in open-world (i.e., test classes are disjoint from training classes) multi-label object recognition when 
clustering images such that each image may contain multiple objects from a fixed set (e.g., the shapes in Figure~\ref{fig:illustration}).
In some scenarios, the \emph{composition} function
that specifies how examples are generated from other examples might be as simple as superposition by element-wise maximum
or addition. However, a more powerful form of composition -- and the main motivation for  our work --
is enabled by \emph{compositional embedding models}, which are a new technique for  few-shot learning.

\begin{figure}
\begin{center}
\includegraphics[width=0.7\textwidth]{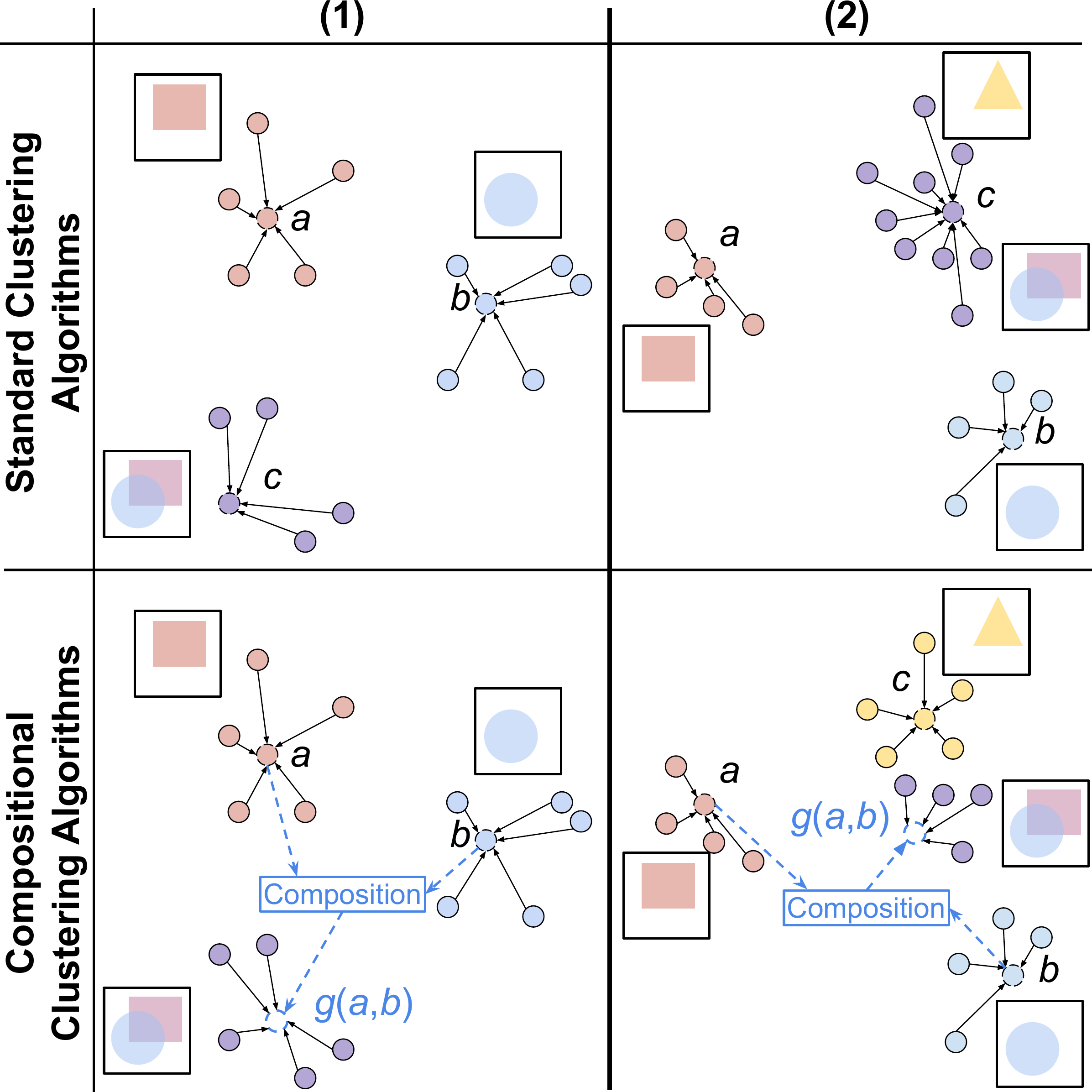}
\end{center}
\caption{{\em Conceptual overview of our paper}:
Scenario (1) shows clusters of images (containing rectangles, circles, or both) and their assigned exemplars (for exemplar-based methods) or centroids (for centroid-based methods) $a$, $b$, $c$, etc. Each arrow represents the assignment of an example to its cluster exemplar/centroid. Standard clustering algorithms such as $k$-means or  Affinity Propagation detect 3 clusters that are independent of each other. Compositional clustering algorithms like CAP, CKM, and GCR can infer that each example in the bottom/purple cluster is \emph{composed} (via $g$) of examples from clusters $a$ \& $b$. Scenario (2) illustrates how modeling compositionality can enable CAP and CKM to find purer clusters by not lumping the two sets of images (some with triangles, and some with rectangles \& circles) together.}
\label{fig:illustration}
\end{figure}
{\bf Compositional embedding models}: 
Standard (non-compositional) embedding models for few-shot learning such as FaceNet \cite{schroff2015facenet} and x-vector \cite{snyder2018x}
have an embedding function $f^\textrm{emb}$ (typically a neural network)
that maps each example (e.g., image, audio clip) into an embedding space so that examples with the same label
are mapped close together, and examples with different labels are mapped far apart.
Compositional embeddings \cite{alfassy2019laso,li2021compositional} go a step further and are trained to
separate not just individual labels, but entire \emph{sets} of labels.
As an example of how this is performed using the approach by \cite{li2021compositional}, suppose an image collection contains some images of rectangles, some of circles, and some of both (see Figure~\ref{fig:illustration}).
Then the embedding function $f^\textrm{emb}$ would induce three clusters in the embedding space
corresponding to $\{ \textrm{rectangle} \}, \{ \textrm{circle} \}$ and $\{ \textrm{rectangle},\textrm{circle} \}$.
In addition to $f^\textrm{emb}$, compositional embedding models 
have a \emph{composition} function $g$ that takes the embedding vectors $x_a,x_b$ of two examples and computes 
a set relationship between them. For instance, $g(x_a, x_b)$ might return another vector in the same embedding
space corresponding to where an example containing the \emph{union} of the labels in the two inputs would lie -- 
see  Figure~\ref{fig:illustration} (lower left). In particular, the training objective is for $g(x_a,x_b) \approx x_{ab}$, where $x_{ab}$ is the embedding of an example containing \emph{both} classes $a$ and $b$. By applying it recursively (e.g., $g(x_a, g(x_b, x_c)) \approx x_{abc}$), the same function $g$ can be used to estimate the embeddings of larger sets of examples as well. 
At test time, $f^\textrm{emb}$ and $g$ are used together (along with a support set of few-shot examples) to infer the \emph{set} of labels
in an example. Other recent works have explored a similar idea of training the embedding network for set operations such as union, difference, and containment \cite{song2021deep,zhou2021multi}, or to use the embedding space to synthesize feature vectors with specific properties \cite{hariharan2017low}.

{\bf Compositional clustering methods}:
In this paper we present and evaluate  three novel algorithms for tackling the ``compositional clustering'' problem: (1) {\bf Compositional $k$-means (CKM)}, which is a centroid-based clustering method; (2) 
{\bf Compositional Affinity Propagation (CAP)}, which is an exemplar-based method; and (3) {\bf Greedy Compositional Reassignment (GCR)}, which can be used in tandem with any standard clustering algorithm.
All three of these methods have the ability to assign each example to either a ``singleton'' cluster corresponding to a single class (e.g., a single speaker, or a single object) \emph{or} to a ``compositional'' cluster corresponding to the union of multiple classes (e.g., a set of speakers, or a set of objects). CKM and CAP have the additional ability to harness the compositional structure of the data to partition them more accurately than is possible with standard clustering algorithms.

As a conceptual illustration, see Figure \ref{fig:illustration}. In
scenario (1) (left half of the figure), there are three sets of images -- some contain circles, some contain rectangles,
and some contain both. Standard clustering algorithms such as Affinity Propagation and $k$-means can separate the data
correctly into three clusters. However, a compositional clustering algorithm such as CKM, CAP, or GCR can also infer that the cluster shown in purple in the bottom-left 
is actually a \emph{compositional cluster} in which each example contains \emph{both} objects from the first two clusters.
Scenario (2) in the figure shows how modeling the compositionality can yield a more accurate partition: whereas
standard clustering algorithms will lump together the images containing triangles with those containing a composition of rectangles and circles, CAP and CKM
can identify this relationship automatically and thereby obtain purer clusters.

{\bf General Workflow}: Here is how a compositional clustering algorithm can be used for open-world object recognition, speaker diarization, and similar tasks: The first step is to (1) train a compositional embedding model \cite{alfassy2019laso,li2021compositional} with both an embedding function $f^\textrm{emb}$ (e.g., with triplet loss, ArcFace loss \cite{deng2019arcface}, etc.)
as well as a composition function $g$ that computes the location in the embedding space corresponding to the \emph{set union} of the classes represented in its two input embeddings. Note that $g$ can be trained recursively \cite{li2021compositionalEmb} to enable the computation of set unions of \emph{arbitrary} size; moreover, it needs to be trained only once and can then be reused. Next, (2) compute the embeddings of all the speaker utterances (or images) in the dataset; we denote the set of these embeddings as $\mathcal{X} = \{ x_1, \ldots, x_n \}$. (3) Pass $\mathcal{X}$, as well as 
the composition network $g$, as input to the compositional clustering algorithm (CAP, CKM, or GCR). The clustering algorithm then (4) infers the cluster label -- which could be either a \emph{singleton} (a single speaker in isolation, or a single object appearing by itself) or a \emph{set} (multiple speakers in simultaneous speech, or multiple objects co-occuring in an image) -- of each example.

{\bf Contributions}: (1) We consider the computational problem of clustering data with compositional structure, particularly as afforded by compositional embedding models, in the setting where (a) the test classes are disjoint from training classes, (b) each example can belong to multiple classes, and (c) no information about the test classes (neither a support set, nor a semantic description) is given.  To our knowledge, this particular task has not been tackled previously. We also define a new accuracy metric, the Compositional Rand Index, for this problem.  (2) We present three novel clustering algorithms -- CAP, CKM, and GCR -- that can partition data and infer their compositional structure automatically. (3) We illustrate how these new methods can infer the clusters, as well as  their compositional relationships, more accurately compared to standard clustering algorithms (Affinity Propagation, $k$-means, Gaussian mixtures, etc.) in two challenging application areas: speaker recognition from speech with multiple overlapping speakers, and multi-label object recognition in open-world scenarios.\footnote{The data and code are available at
\url{https://github.com/jwhitehill/CompositionalClustering}.
}


\section{Related Work}

\subsection{Multi-Label Few-Shot and Zero-Shot Learning}
The past 5 years have seen significant growth in the fields of multi-label few-shot and zero-shot learning (e.g., \cite{lee2018multi,narayan2021discriminative,chen2020knowledge,huynh2021interaction}). Much of this work relies on the existence of a knowledge graph such as WordNet \cite{miller1995wordnet}, a word embedding space such as GloVe \cite{pennington2014glove}, or external attribute vectors, to represent relationships among classes and thereby enable the model to generalize to data from unseen classes at test time. In contrast, the compositional embedding models of \cite{alfassy2019laso} and \cite{li2021compositional}, and thus our work as well, make no such assumption -- each class can be completely independent of each other. 
To our best knowledge, no prior work has investigated how to \emph{cluster} examples automatically when the test classes are disjoint from training classes, when no support sets are provided, and when no semantic information about the test classes is provided. (Note that, when few-shot examples are provided for the test classes, then the ``clustering'' problem becomes trivial -- the examples can be grouped based just on their estimated label vectors.)

\subsection{Clustering}
To our best knowledge, no previous clustering algorithm   can both cluster a dataset and infer the compositionality
among clusters. (A recent paper \cite{pal2022clustering} examines how to cluster data that is ``compositional'' in the  sense that they lie on a simplex and thus the features within every example ``compose'' to 1, but this is very different from our scenario.)
Below we discuss the most similar work.

{\bf Mixture models}, such as the Mixture of Gaussians fit using Expectation-Maximization, the Dirichlet mixture process \cite{blei2006variational}, and the fuzzy $k$-means clustering algorithm \cite{bezdek1984fcm},
extend the standard $k$-means  algorithm by ``softly'' assigning each data point
to a probability distribution over the mixture components instead of giving a ``hard'' assignment like in $k$-means.
Importantly, these approaches assume that each data point is generated by a
\emph{single} cluster, and the probability distribution expresses the uncertainty over which cluster it is.
They can capture compositionality only in a limited sense by assuming that examples that lie between two (or more) cluster centroids belong to both (or all) of these clusters. These methods cannot
distinguish between an example that is unconfidently assigned to a single
cluster (thus resulting in high entropy over the mixture components for that example),
from an example that is confidently assigned to multiple clusters. Moreover, they will fail if the compositional cluster (e.g., the purple cluster in the left half of Figure \ref{fig:illustration}) does not lie near the mean of its constituent singleton clusters (the red and blue clusters in the figure).

{\bf Hierarchical clustering algorithms} create a tree (dendrogram) such that 
the $n$ leaf nodes correspond 1-to-1 to the examples in the dataset,
and each internal node $i$ represents a cluster 
whose members consist of all the leaf nodes descending from $i$. Internal nodes closer to the root
correspond to higher-level abstractions of the data.
Hierarchical clustering algorithms can work either
top-down by splitting clusters or bottom-up by merging clusters, until some clustering criterion
is reached. One popular variant is
\emph{Agglomerative Clustering} using the \citet{ward1963hierarchical} criterion,
which seeks to minimize the variance within each cluster. In all cases, hierarchical clustering algorithms
assign each example to a sequence of clusters of increasing generality, starting from the internal node just above  the
leaf all the way up to the root node, such that each parent cluster 
captures the \emph{intersection} of the characteristics of the child clusters. In contrast, our
proposed method can assign each example to contain the \emph{union} of the properties in multiple clusters;
this is tantamount to a dendrogram where each example is connected by an edge to multiple parent nodes, thus
yielding a directed acyclic graph rather than a tree.

{\bf Multi-view clustering algorithms} (e.g., \citet{bickel2004multi})
partition the feature space into
multiple subsets, each corresponding to a different ``view'' of the data (see \cite{yang2018multi,fu2020overview} for recent surveys). For instance, each example might
be a video and thus have both auditory and visual features associated with it. Since multiple views often contain complementary information, harnessing all of them can often improve clustering accuracy. Moreover,
the structure of the data
from one view can provide implicit supervision when clustering using the other views. However, existing multi-view clustering methods do not
have the ability to model compositionality. \citet{franklin2018compositional} recently proposed
a method for ``compositional clustering in task structure learning'', but their method is more akin
to multi-view clustering, and the compositionality pertains to how they tackled a control problem
(separately addressing the reward and transition functions), not the clustering problem itself.

{\bf Exemplar-based} clustering algorithms differ from {\bf centroid-based} algorithms in how clusters are represented:
In the former, each cluster is represented by a specific example in the dataset; in contrast, the latter  (e.g., $k$-means)
may compute a function of the examples (e.g., the mean) to represent the cluster. One of the mostly widely used
exemplar-based clustering algorithms is Affinity Propagation \cite{frey2007clustering}.

\section{Approach I: Compositional Affinity Propagation (CAP)}
Our first novel algorithm is Compositional Affinity Propagation, which is an exemplar-based clustering method and based on standard Affinity Propagation (AP) algorithm that is widely used for speaker diarization to group clusters of  utterances into distinct speakers \cite{yin2018neural,li2021compositionalEmb}. CAP is based on an undirected probabilistic graphical model whose likelihood is approximately optimized using discrete optimization. Before presenting CAP,  we first review  standard AP \cite{frey2007clustering}.

\subsection{Review of Affinity Propagation}
Let $\mathcal{X} = \{ x_1, \ldots, x_n \} \subset \mathbb{R}^p$ be a dataset, and let
$\mathcal{C} = {\{1,\ldots,n\}} \doteq [n]$ be the set of indices of the (embedded) examples in $\mathcal{X}$.
Next, let $c_1,\ldots, c_n \in \mathcal{C}$  be the cluster assignments: Each $c_i$ denotes the \emph{exemplar} representing
the cluster to which example $i$ belongs; if example $i$ itself is the exemplar for some other example $j\ne i$, then we require $c_i=i$.
For instance, if $\mathcal{X}$ contains $n=3$ examples, the first two of which
belong to the same cluster and the third of which belongs to its own cluster,
then we might have $c_1=2,c_2=2,c_3=3$ (or possibly $c_1=1,c_2=1,c_3=3$).
Let $S: \mathcal{C} \times \mathcal{C} \rightarrow [-\infty, 0]$ map from a pair
of example indices to the negative (squared) distance between the examples, i.e., $S(i,j) = - \| x_i - x_j \|^2$ for $i \ne j$; and let $S$ map to a constant value for $i=j$, i.e., $S(i,i)=\gamma$, where $\gamma$  is the ``preference'' 
(a hyperparameter)
that $x_i$ is an exemplar, where  larger negative values discourage those examples from becoming exemplars.
From these definitions, we can formulate the following constrained optimization problem:
\[
\argmax_{c_1,\ldots,c_n \in \mathcal{C}} \sum_{i=1}^n S(i,c_i) \qquad \textrm{s.t.}\quad (\exists i:c_i=k) \implies c_k=k
\]
The objective is the sum of  distances between each point and its assigned exemplar, and the constraints
enforce consistency that examples used by others as exemplars also designate themselves as exemplars.
The optimization  has to weigh the cost $\gamma$ of creating a new cluster
against assigning examples to existing exemplars that are farther away.

{\bf Illustration}: Given an appropriate choice for the $\gamma$, Affinity Propagation would yield the results
shown in the top half of Figure \ref{fig:illustration}. In particular, in scenario (1), the cluster shown in purple would be identified
as an independent cluster with examplar $c$, and in scenario (2), the cluster shown in purple would contain the images
with triangles as well as those composed of rectangles  and circles.

{\bf Inference}: \citet{frey2007clustering} showed a procedure 
to find approximately optimal solutions by defining
a \emph{factor graph}  
to represent the variables and constraints, 
where $S$  is interpreted as containing log-likelihoods,
and then applying loopy belief propagation. This results in a new optimization problem where the goal is to find
maximum a posteriori (MAP) solutions to
$\argmax_{c_1,\ldots,c_n \in \mathcal{C}} P(c_1, \ldots, c_n\ |\ S)$.
where probability distribution $P$ is understood to encode the constraints.
Specifically, the factor graph contains variable nodes to represent $c_1,\ldots,c_n$
and factor nodes to represent both the log-likelihoods $S(1,\cdot), \ldots, S(n,\cdot)$ and a set of \emph{constraints} $\delta_1,\ldots, \delta_n$. Each $\delta_k$ encodes whether $c_k$ is compatible with the other $c_{k'\ne k}$:
\begin{equation}
\label{eqn:delta}
\delta_k(c_1,\ldots,c_n)= \left\{ \begin{array}{cl}- \infty & \textrm{if}\ \exists i: (c_i=k) \wedge (c_k\ne k) \\
                                                     0 & \textrm{otherwise}
				   \end{array} \right.
\end{equation}
Given the factor graph, a sequence of ``messages'' (functions $\alpha, \rho: [n]\times[n]\times \mathcal{C} \rightarrow  [-\infty, 0]$) is passed back and forth between the variable and factor nodes.
Each variable $i$ sends a message $\R{i}{k}{c_i}$ to constraint $k$, and
each constraint $k$ sends a message $\A{i}{k}{c_i}$ to variable $i$, about the likelihood of each possible value of $c_i$.
The values of $\alpha$ and $\rho$ are determined by the max-product algorithm for loopy belief propagation \cite{weiss2001optimality} applied to
the factor graph (see the Appendices).
To find an approximate MAP estimate for all the $c_i$, we  alternate between 
computing the $\alpha$'s and the $\rho$'s. Finally, after any number of iterations, we compute $c_i^\textrm{MAP} = \argmax_{c_i} \left[ \sum_k \A{i}{k}{{c_i}} + S(i,{c_i}) \right]$.
\citet{frey2007clustering} also presented an efficient ($O(n^2)$) method to calculate all the messages for each iteration.

\subsection{Procedure: Compositional Affinity Propagation}





Here we describe our proposed Compositional Affinity Propagation algorithm. At a high level,  CAP innovates on classic AP by allowing each cluster to be represented by not just a single example (``singleton’’ cluster), but rather an entire \emph{set} of examples (``compositional’’ cluster). Importantly, the examples in this set need not be semantically similar or lie close to each other in the feature space; rather, the \emph{union} of the characteristics of the examples in this set should be present in \emph{each} of the examples belonging to the compositional cluster. In terms of the inference procedure, CAP is somewhat more complex than standard AP due to the need, as part of the max-product algorithm, to compute the maximum values of many subsets efficiently (FindAllMaxes).

Let $\mathcal{X} = \{ x_1, \ldots, x_n \} \subset \mathbb{R}^p$ be a dataset. Let
$\mathcal{C} \subset 2^{[n]} \setminus \emptyset$ be the set of compositions of examples in $\mathcal{X}$ under consideration, where
we assume $\mathcal{C}$ contains all the singletons, i.e., $\{i\} \in \mathcal{C}, i=1,\ldots, n$.
Let $d = \max_{c \in \mathcal{C}} |c|$, i.e., the size of the largest composition under consideration. 
To identify which compositions contain (or do not contain) each example $k$, define functions
$\phi, \overline{\phi}: [n] \rightarrow 2^\mathcal{C}$
such that $\phi(k) = \{ c \in \mathcal{C}: c \ni k \}$ and $\overline{\phi}(k) = \mathcal{C} \setminus \phi(k)$.

Let $f$ be defined as in standard Affinity Propagation. We further assume there is a function $g: 2^\mathcal{X} \setminus \emptyset \rightarrow \mathbb{R}^p$ that consumes
a set of examples and produces another vector representing  their composition; for singleton sets, we let $g$ be the identity function, i.e.,
$g(\{ x \})=x$. For instance, $g$ could be the element-wise
maximum to perform pixel-wise superposition of the images;  for word embeddings, it could be element-wise addition \cite{allen2019analogies};
or it could be a trained neural network within a compositional embedding model.  We define 
$S: [n] \times \mathcal{C} \rightarrow [-\infty,0]$ to measure the  distance between each example and each composition:
$S(i,c)=- \| x_i - g(\{ x_k: k \in c \}) \|$ for $c \ne \{i\}$,
and $S(i,\{i\})=\gamma$ is a hyperparameter for each example.

Finally, define  $c_1,\ldots, c_n \in \mathcal{C}$ as the assignment of which example belongs to which cluster.
If $c_i=\{k\}$ (i.e., a singleton), then example $i$ belongs to a \emph{singleton} cluster with exemplar $x_k$.
If $|c_i|\geq 2$, then example $i$ belongs to the cluster with a \emph{compositional exemplar} $g(\{x_k:k \in c_i\})$,
i.e., the composition of all the examples in $c_i$. Note that, in general, 
compositional exemplars are not members of $\mathcal{X}$.  In CAP, we require that, 
whenever some example $i$ designates its exemplar either to \emph{be} example $k$ ($c_i=\{k\}$), or to \emph{include}
example $k$ ($c_i \ni k$), then example $k$ must designate itself as an exemplar ($c_k = \{ k \}$).
Example: if $\mathcal{X}=\{x_1,x_2,x_3,x_4\}$ and we allow compositions of size at most 2, then
$\mathcal{C}=\{\{1\}, \{2\}, \{3\}, \{4\}, \{1,2\}, \{1,3\}, \{1,4\}, \{2,3\}, \{2,4\},\{3,4\}\}$; if $x_1$ and $x_2$
each constitutes its own cluster and the last two examples are both assigned to the composition
of the first two clusters, then we would have $c_1=\{1\}, c_2=\{2\}, c_3=c_4=\{1,2\}$.

Our new constrained optimization problem is thus:
\[
\argmax_{c_1,\ldots,c_n \in \mathcal{C}} \sum_{i=1}^n S(i,c_i)\quad \textrm{s.t.}\quad (\exists i:c_i\ni k) \implies c_k=\{k\}
\]
Importantly, the optimization objective incurs \emph{no additional cost} when an example is assigned to a compositional exemplar $c_k$ as long as
all of the examples $k' \in c_k$ have themselves already been designated as exemplars.

{\bf Illustration}: Given an appropriate choice for $\gamma$, CAP
would yield the results
show in the bottom half of Figure \ref{fig:illustration}. In scenario (1), the cluster shown in purple would be identified
as a cluster with a \emph{compositional exemplar}. In scenario (2), 
the compositional structure identified by the algorithm could help it to separate the images containing triangles from those
containing both a rectangle and a circle.

\subsubsection{Inference}
As with standard Affinity Propagation, we find a MAP estimate for each $c_i$ by defining a factor graph and
computing and passing messages between the variables and the factors. 
We adjust the definition of $\delta_k$ to be:
\begin{equation}
\label{eqn:delta2}
\delta_k(c_1,\ldots,c_n)= \left\{ \begin{array}{cl}- \infty & \textrm{if}\ \exists i: (c_i\ni k) \wedge (c_k\ne \{k\}) \\
                                                     0 & \textrm{otherwise}
				   \end{array} \right.
\end{equation}
In the Appendices, we derive a procedure to compute  $\alpha$ and $\rho$ efficiently; see  
Algorithm \ref{alg:cap}. 

\begin{theorem}
Let $n$ be the number of examples in a dataset $\mathcal{X}$, and let $d$ be the largest 
element in the set $\mathcal{C}$ containing all compositions under consideration.
Then Algorithm 1 implements message passing (i.e., computation of sufficient statistics of $\alpha$ and $\rho$)
for Compositional Affinity Propagation and operates in time $O(dn^{d+1})$ per iteration.
\end{theorem}
\begin{proof}
See Appendices.
\end{proof}


\begin{algorithm}
\begin{algorithmic}
\caption{Compositional Affinity Propagation (CAP)}\label{alg:cap}
\STATE \hspace{-.4cm} $\textrm{\bf CAP}(S, \mathcal{C})$:
  \STATE $\phi(k) \gets \{ c \in \mathcal{C}: c \ni k \},\quad \overline{\phi}(k) \gets \mathcal{C} \setminus \phi(k) \quad \forall k$
  \STATE $q(i,c_i)\gets 0 \quad \forall i,c_i$
  \STATE $a(i,k)\gets 0, \quad \overline{a}(i,k)\gets 0  \quad \forall i,k$
  \WHILE{not converged}
    \STATE $b,\overline{b},h \gets \textrm{ComputeRhoStats}(S, \mathcal{C}, \phi, \overline{\phi}, a, \overline{a}, q)$
    \STATE $a,\overline{a},q \gets \textrm{ComputeAlphaStats}(\mathcal{C}, b,\overline{b},h)$
  \ENDWHILE
  \STATE  {\bf return} $\argmax_{c_i} (q(i,c_i) + S(i,c_i)) \quad \forall i$

\vspace{.2cm}
\STATE \hspace{-.4cm} $\textrm{\bf ComputeRhoStats}(S, \mathcal{C}, \phi, \overline{\phi}, a,\overline{a},q)$:
  \FOR{$i=1,\ldots,n$}
          \STATE $r,s \gets \textrm{FindAllMaxes}(S(i,\cdot) + q(i,\cdot), \mathcal{C}, \phi, \overline{\phi})$
          \FOR{$k=1,\ldots, n$}
            \STATE $b(i,k) \gets \max(r(k) - a(i,k), s(k) - \overline{a}(i,k))$
            \STATE $\overline{b}(i,k) \gets s(k) - \overline{a}(i,k)$
          \ENDFOR
  \ENDFOR
  \FOR{$k=1,\ldots,n$}
    \STATE $h(k) \gets S(k,\{k\}) + q_k(\{k\}) - a(k,k)$
  \ENDFOR
  \STATE {\bf return} $b,\overline{b},h$
%
\vspace{.2cm}
\STATE \hspace{-.4cm} $\textrm{\bf ComputeAlphaStats}(\mathcal{C}, b,\overline{b},h)$:
  \FOR{$k=1,\ldots, n$}
    \STATE $e(k) \gets \sum_{i' \ne k} b(i',k)$
    \STATE $\overline{e}(k) \gets \sum_{i' \ne k} \overline{b}(i',k)$
  \ENDFOR
  \FOR{$i=1,\ldots,n$}
    \FOR{$k=1,\ldots, n$}
      \IF{$i=k$}
        \STATE $a(i,k) \gets e(k)$
        \STATE $\overline{a}(i,k) \gets \overline{e}(k)$
      \ELSE
        \STATE $a(i,k) \gets h(k) + e(k) - b(i,k)$
        \STATE $\overline{a}(i,k) \gets \max(\overline{b}(k,k) + \overline{e}(k) - \overline{b}(i,k),\  h(k) + e(k) - b(i,k))$
      \ENDIF
    \ENDFOR
  \ENDFOR
  \FOR{$i=1,\ldots,n$}
    \STATE ${q}^*(i) \gets \sum_{k'} \overline{a}(i,k')$
    \FOR{$c_i\in \mathcal{C}$}
      \STATE $q(i,c_i) \gets {q}^*(i) + \sum_{k' \in c_i} (a(i,k') - \overline{a}(i,k'))$
    \ENDFOR
  \ENDFOR
  \STATE  {\bf return} $a,\overline{a},q$
%
\end{algorithmic}
\end{algorithm}

\begin{algorithm}
\begin{algorithmic}
\caption{Finding Maxima of Many Subsets}\label{alg:findmaxes}
\vspace{.2cm}
\STATE \hspace{-.4cm} $\textrm{\bf FindAllMaxes}(q, \mathcal{C}, \phi, \overline{\phi})$:
\STATE $r(k) \gets \max q(\phi(k))\quad \forall k$
\STATE $s(k) \gets -\infty\quad \forall k$
\FOR{$j = 1, \ldots, d$}
  \FOR{$\tau=\{t_1,\ldots,t_{j-1}\}\ \textrm{s.t.}\ \exists t_j>t_{j-1}: \{t_1,\ldots,t_j\} \in \mathcal{C}$}
    \STATE $\psi_\tau\gets\{\{t_1,...,t_{j-1},t_j\} \}_{t_j>t_{j-1}} \cap \mathcal{C}$
    \STATE $c^1, c^2 \gets \argmax^{1,2}_{c \in \psi_\tau} q(c)$
      \FOR{$k=1,\ldots,n$}
        \IF{$c^1 \in \overline{\phi}(k)$}
          \STATE $s(k) \gets \max(s(k), q(c^1))$
        \ELSIF{$c^2 \in \overline{\phi}(k)$}
          \STATE $s(k) \gets \max(s(k), q(c^2))$
        \ENDIF
      \ENDFOR
  \ENDFOR
\ENDFOR
\STATE  {\bf return} $r,s$
\end{algorithmic}
\end{algorithm}

{\bf Inferring the Number of Clusters}:
The hyperparameter $\gamma$ in CAP is the penalty for creating a new cluster versus assigning data points to an existing one. It is similar to the concentration hyperparameter (often denoted $\alpha$) in the Dirichlet mixture process \cite{blei2006variational}.

\subsection{CAP$\subset$: An Approximation to CAP}
To improve the scalability of CAP, we can apply it
to a randomly selected subset of examples $\widetilde{\mathcal{X}} \subset \mathcal{X}$
and infer the cluster assignments $\tilde{c}_1,\ldots,
\tilde{c}_{|\widetilde{\mathcal{X}}|}$. Let $\mathcal{E} = \{ \tilde{c}_i  \}_{i=1}^{|\widetilde{\mathcal{X}} |}\subset \mathcal{C}$ be the set of \emph{unique} exemplars (singleton or compositional) inferred for $\widetilde{\mathcal{X}}$.
Then, for each example $x_i$ in the original dataset $\mathcal{X}$, we designate its
exemplar to be the $\tilde{c}_i \in \mathcal{E}$ that is closest to it $x_i$ according to $f$.
Specifically, we assign $c_i = \argmin_{\tilde{c}_i \in \mathcal{E}} \| x_i, g(\{ x_j: j \in \tilde{c}_i \}) \|$.
We call this method CAP$\subset$. 

\section{Approach II: Compositional $k$-means}
The second compositional clustering algorithm we propose is called Compositional $k$-means (CKM). In contrast to CAP, which uses discrete optimization to assign examples to exemplars, CKM uses gradient descent to minimize a sum of squared distances by adjusting the real-valued cluster centroids. Like CAP, the CKM method can potentially cluster the data in Figure \ref{fig:illustration} more accurately by harnessing the composition function $g$ to infer which examples belong to singleton clusters versus compositional clusters. CKM is a centroid-based method rather than an exemplar-based clustering method. Hence, each cluster assignment variable $c_i$ is a subset of $[k]$ (rather than of $[n]$, like with CAP).

\subsection{Review of classic $k$-means}
Given the number of clusters $k$ as input, classic $k$-means seeks to assign each of the $n$ examples to one of the $k$ clusters (denoted $c_i \in [k]$ for each $i)$, so as to minimize the  sum of squared distances (SSD)
    $\textrm{SSD}(\{ m_j \}_{j=1}^k, \{ c_i \}_{i=1}^n) = \sum_{i=1}^n \|x_i - m_{c_i}\|^2$.
Here, each $m_j \in \mathbb{R}^p$ is a cluster centroid, and each $c_i \in [k]$ is a cluster index. To (locally) minimize the SSD, two steps are executed in alternation until convergence:
\begin{enumerate}
    \item  Assign each $x_i$ to the cluster $j$ whose centroid $m_j \in \mathbb{R}^p$ is closest to $x_i$; and
    \item Compute each centroid $m_j$ as the mean of the points assigned to cluster $j$.
\end{enumerate}
In particular, the second step is the closed-form minimizer of the SSD w.r.t.~the centroids $m_j$.
Since each of these steps is guaranteed not to increase the SSD,
and since a lower bound on SSD is always 0, the algorithm is guaranteed to converge to a local minimum.

\subsection{Procedure: Compositional $k$-means}
\label{sec:ckm}
Let the number of singleton clusters $k$ (e.g., the number of individual speakers in the audio, or the number of basic object classes in the image set) be known, and let $\mathcal{K} \subset 2^{[k]}$ be a set of possible compositions of the singleton clusters, where we require that $\mathcal{K}$ contains all the singletons, i.e., $\{ i \} \in \mathcal{K}, i=1,\ldots,k$. Assume composition function $g$ is differentiable. CKM seeks to assign each $x_i$
to \emph{either} one of $k$ singleton clusters (a single person speaking in isolation, or a single object by itself) \emph{or} to a compositional cluster (the composition of multiple speakers in an audio, or multiple objects in an image) so as to minimize the following sum of squared distances:
\begin{equation}
    \label{eqn:ssd_ckm}
    \textrm{SSD}(\{ m_{\{j\}} \}_{j=1}^k, \{ c_i \}_{i=1}^n) = \sum_{i=1}^n \|x_i - m_{c_i}\|^2
\end{equation}
where each compositional centroid $m_{\eta}=g(\{ m_{ \{j\} } \}_{j\in \eta})$ (for $\eta \in \mathcal{K}$ and $|\eta|>1$) is computed using the composition function $g$. (Note the small difference in notation compared to the SSD in standard $k$-means in the subscript of $m$ so as to emphasize that a cluster centroid may represent the composition of other clusters.)

Like the classic $k$-means, the SSD is a function of the \emph{singleton} cluster centroids (i.e., $m_{\{1\}}, \ldots, m_{\{k\}}$). Unlike classic $k$-means, the CKM method can assign each example to either a singleton or a compositional cluster. By adjusting  the singleton centroids, the locations of the \emph{compositional} centroids -- and thus the SSD value itself -- are also affected due to their dependence via $g$.

At a high level, CKM works as follows: After initializing the singleton cluster centroids randomly and computing the compositional centroids using $g$, a two-step alternating procedure is executed whereby (a) each example $x_i$ is assigned to the closest centroid (either singleton or compositional), and (b) the \emph{singleton} centroids $m_{\{1\}}, \ldots, m_{\{k\}}$ are adjusted using gradient descent (with learning rate $\epsilon$) to reduce the SSD in Equation \ref{eqn:ssd_ckm}. Since  we assume $g$ is a differentiable function (typically implemented as a neural network), the gradient of the SSD w.r.t.~each singleton centroid (keeping the weights of $g$ fixed) can be computed easily. During the optimization (see Algorithm \ref{alg:ckm}), CKM  dynamically infers which clusters are singletons and which are compositional, and also estimates the centroids of the singleton clusters so as to trade off between fitting the singletons and  the compositional clusters well. Note that (like with classic $k$-means) the initialization in step 1 can affect which local minimum is reached, and thus it is often useful to
try multiple random seeds and to choose the best seed based on the lowest SSD. 

\begin{algorithm}
\begin{algorithmic}
\caption{Compositional $k$-means (CKM)}\label{alg:ckm}
\STATE \hspace{-.4cm} $\textrm{\bf CKM}(\mathcal{X}, \mathcal{K}, \epsilon)$:
  \STATE Set each $m_{\{j\}}, j \in [k]$ to a randomly drawn (without replacement) example in $\mathcal{X}$.
  \STATE Compute compositional centroids: $m_{\eta} \gets g(\{ m_{ \{j\} } \}_{j\in \eta})\quad \forall \eta \in \mathcal{K}: |\eta|>1$.
  \WHILE{not converged}
    \STATE $c_i \gets \argmin_{\eta \in \mathcal{K}} \|x_i - m_\eta\|^2\quad \forall i$.
    \STATE $m_{\{1\}},\ldots,m_{\{k\}} \gets \textrm{SGD}\left( \sum_{i=1}^n \|x_i - m_{c_i}\|^2; \{ m_{\{j\}} \}_{j=1}^k; \epsilon \right)$
  \ENDWHILE
  \STATE {\bf return} $\{ c_i \}_{i=1}^n$
\end{algorithmic}
\end{algorithm}


{\bf Convergence}:
Unlike in classical $k$-means,
the second step of the alternation (adjustment of cluster centroids) in CKM is conducted numerically rather than analytically. However, assuming  the learning
rate of gradient descent is sufficiently small, it will not increase the SSD. Since the first step of the alternation can also never increase the SSD, and since the SSD is bounded below, the algorithm will converge to a local minimum.

{\bf Comparison to CAP}: CKM requires that $g$ be differentiable, and it uses gradient-based optimization using neural network packages such as TensorFlow, PyTorch, etc. In contrast, CAP, as it only involves computing maxes and sums, can be implemented in simple Python or C code, and it does not require a differentiable $g$. Each step of the while-loop takes runtime $O(nk^d)$, where $d$ is the size of the largest composition in $\mathcal{K}$. Since the number of singleton clusters $k$ is typically much smaller than the number of examples $n$, CKM can run much faster than CAP.

{\bf Inferring the Number of Clusters}:
While CKM takes the number of singleton clusters $k$ as input, it infers the number of compositional clusters automatically based on the data -- if no examples are assigned to a particular composition, then that compositional cluster does not exist. Moreover, the value $k$ itself can be estimated by techniques such as the Gap statistic \cite{tibshirani2001estimating} that is commonly used for standard $k$-means clustering.

\section{Approach III: Greedy Compositional Reassignment (GCR)}

The third approach that we explored for compositional clustering is based
on the idea  of using any standard clustering algorithm to partition the data $\mathcal{X}$ into clusters, and then using the composition function
$g$ to find the optimal ``reassignment'' of the inferred clusters so that some of them are considered to be compositions of others. Suppose
we first obtain (e.g., from Agglomerative Clustering) a set $\mathcal{E} = \{ m_{\{1\}}, \ldots, m_{\{k\}} \}$ of  $k$ cluster centroids. Then we could iterate over every possible subset $\mathcal{\tilde E} \subseteq \mathcal{E}$; these represent the compositional
clusters. For each $\mathcal{\tilde E}$, we conduct an inner-loop
to iterate over every possible 1-to-1 map from $\mathcal{\tilde E}$ to the set of compositions (via $g$) of $\mathcal{E} \setminus \mathcal{\tilde E}$;
these are the singleton/singleton clusters.
We would finally select $\mathcal{\tilde E}$ and its map to $\mathcal{E} \setminus \mathcal{\tilde E}$ so as to minimize
the sum of distances between the examples and their assigned cluster centroids (either singletons or compositional).

Unfortunately, due to the factorial time cost, this brute-force approach quickly becomes completely impractical (e.g., for $|\mathcal{E}|=15$ and $d=2$,
there are $107770296705436$ possibilities). However,
the idea gave us inspiration for a tractable greedy heuristic 
that we call Greedy Compositional Reassignment (GCR). Like CKM, GCR is a centroid-based clustering method. It uses $g$ and the distances between cluster centroids  to determine the compositional relationships in a greedy manner and thereby avoid the factorial time cost. 



\subsection{Procedure: Greedy Compositional Reassignment}
Assume that a standard clustering method (we use Agglomerative Clustering) has produced a clustering with $k$ centroids $m_{\{1\}},\ldots, m_{\{k\}}$ and cluster assignments $c_1, \ldots, c_n$, where each $c_i \in [k]$. Let $\mathcal{K}$ be the set of compositions under consideration.
GCR first uses $g$ to compute the location of the compositional centroid for every $\eta \in \mathcal{K}$.
It then finds, for each putative singleton cluster $j \in [k]$, the distance $d_j$ to the closest compositional centroid $b_j \in \mathcal{K}$; if $d_j$ is below a threshold $\tau$, then cluster $j$ is concluded to actually be a composition of the two other clusters in $b_j$, and all the examples that were previously assigned to cluster $j$ are reassigned to the compositional cluster $b_j$. The process is repeated  for each singleton cluster $j$ according to the distances $d_j$ sorted from smallest to largest until one of several possible termination conditions are reached (so as to maintain consistency, e.g., avoid cycles of compositionality); once this point is reached, all the remaining clusters that were not reassigned to be compositional are deemed to be singletons. The algorithm uses sets $\mathcal{S}$ and $\mathcal{T}$  to keep track of which clusters have been assigned as singletons and which are assigned as compositions, respectively. The final assignment of an example to a cluster index is denoted $c'_i \in \mathcal{K}$ for each example $i \in [n]$. See Algorithm \ref{alg:GCR} for details.

\begin{algorithm}
\begin{algorithmic}
\caption{Greedy Compositional Reassignment (GCR)}\label{alg:GCR}
\STATE \hspace{-.4cm} $\textrm{\bf GCR}(\mathcal{X}, \mathcal{K}, \tau)$:
  \STATE $\mathcal{S} \gets \emptyset,\quad \mathcal{T} \gets \emptyset$.
  \STATE Obtain preliminary clustering: $\{ m_{\{j\}} \}_{j=1}^k, \{ c_i \}_{i=1}^n \gets \begin{tt}AgglomerativeClustering\end{tt}(\mathcal{X})$.
  \STATE Compute compositional centroids: $m_\eta \gets g(\{ m_{\{j\}} \}_{j \in \eta})\quad \forall \eta \in \mathcal{K}$.
  \STATE $b_j \gets \argmin_{\eta} \|m_{\{j\}} - m_\eta\| \quad \forall j \in [k]$.
  \STATE $d_j \gets \|m_{\{j\}} - m_{c_j}\|\quad \forall j \in [k]$.
  \FOR{$j \in [k]$ according to $\begin{tt}argsort\end{tt}(\{d_j\})$}
    \IF{$d_{j} \geq \tau$ or $j \in \mathcal{S}$ or $b_j \cap \mathcal{T} \ne \emptyset$}
      \STATE Assign remaining clusters $j'>j$ to  singletons:
      $c'_i \gets \{j'\} \quad \forall i: c_i=j', j'\geq j$.
      \STATE {\bf break}
    \ENDIF
    \STATE Assign cluster $j$ to  composition: $c'_i \gets b_j \quad \forall i: c_i=j$.
    \STATE Add the clusters in $b_j$ to the set of singletons: $\mathcal{S} \gets \mathcal{S} \cup b_j$.
    \STATE Add the current cluster ($j$) to the set of compositions: $\mathcal{T} \gets \mathcal{T} \cup \{ j \}$.
  \ENDFOR
  \STATE {\bf return} $\{ c'_i \}_{i=1}^n$.
\end{algorithmic}
\end{algorithm}
 




{\bf Comparison to CAP and CKM}: Excluding the runtime cost of the initial clustering using Agglomerative Clustering,  GCR method is much faster than CAP and CKM since it iterates over the $k$ singleton clusters at most once, and each iteration is  simple. 
Note  that, whereas CKM and CAP can use the compositionality to partition the data into clusters more cleanly (see the right half of Figure \ref{fig:illustration}), GCR cannot -- it only has the ability to infer the compositional relationships among already-formed clusters (left half of Fig.~\ref{fig:illustration}).

{\bf Inferring the Number of Clusters}:
Since GCR first runs a standard clustering algorithm as a subroutine, then any technique that can estimate the number of clusters for that clustering algorithm (e.g., Gap statistic) can also be used for GCR.

\section{Experiments}
\label{sec:experiments}
To evaluate the proposed algorithms, we conducted  experiments, using standard datasets that are widely used for few-shot learning research, on both multi-object image recognition and multi-person speaker diarization with overlapping speech. We follow the workflow described in the Introduction, i.e., for each problem domain, we use few-shot learning to train an embedding function $f^\textrm{emb}$ to separate examples by their classes, as well as a composition function $g$ \cite{alfassy2019laso,li2021compositional} that can estimate the location in the embedding space of the union of multiple sets of classes. We train these models jointly and episodically, where the episodes contains examples from different set of classes.
{\bf Evaluation}: Since standard clustering metrics such as the 
Adjusted Rand Index (ARI) do not capture compositionality, we   devise a new  evaluation metric called the Compositional Rand Index. 

\subsection{Evaluation Metric: the Compositional Rand Index (CRI)}
\label{sec:cri}
Suppose  dataset $\mathcal{X}=\{x_1,\ldots, x_n\}$ contains $l$ singleton clusters and some number (possibly 0) of compositional clusters.
Then 
$\mathcal{Y} = 2^{[l]} \setminus \emptyset$ is the set of all possible  ground-truth cluster labels, and $y_1,\ldots,y_n \in \mathcal{Y}$
are the  cluster assignments. A sensible evaluation criterion of some inferred labels $c_1,\ldots,c_n \in \mathcal{C}$
w.r.t.  ground-truth should capture the \emph{number} of clusters, their \emph{purity},
and their \emph{compositional relationships}. It should
\emph{not} depend on the particular naming of cluster labels
(the identifiability issue). With these goals in mind, we propose the
Compositional Rand Index (CRI)
to compute the probability, over all pairs $i\ne j$, that 
the inferred labels agree with the ground-truth about whether the cluster assignment of example $i$ subsumes
the cluster assignment of example $j$:
\begin{eqnarray}
\textrm{CRI}(c_1,\ldots,c_n, y_1,\ldots,y_n)
=
\frac{1}{n(n-1)} \sum_{i\ne j} \mathbb{I}[\mathbb{I}[c_i \supseteq c_j] = \mathbb{I}[y_i \supseteq y_j]]
\end{eqnarray}
where $\mathbb{I}[\cdot]$ is a 0-1 indicator function. For datasets without compositionality (i.e., $\mathcal{Y}=[l]$), CRI is equivalent to the standard Rand Index \cite{rand1971objective}.

\subsection{Baseline Methods}
We 
chose several baselines that seemed the most reasonable alternative approaches, even if they had no explicit ability to model compositionality.

{\bf Ignore compositionality}: One approach is simply to ignore the compositional relationships among clusters and consider each cluster as completely independent; this is illustrated in Figure \ref{fig:illustration} (top row). Any standard clustering method can thus be used. While it will pay a penalty under the CRI metric since it misses the compositional relationships, it can sometimes (Figure \ref{fig:illustration} upper-left, but not upper-right) still do a good job overall by correctly forming coherent clusters. With this motivation, we use standard {\bf Affinity Propagation (AP)} as well as {\bf Agglomerative Clustering (AC)} with the Ward criterion as two baselines. 

{\bf Infer compositionality from ``soft'' label assignments}: Mixture models such as the classic Mixture of Gaussians and fuzzy $k$-means \cite{bezdek1984fcm} (it is more commonly called ``fuzzy $c$-means'' in the literature) assign to each example a vector of probabilities that express the likelihood that it belongs to each of the $k$ clusters. By thresholding these probabilities with some threshold $\tau$, one can obtain a \emph{set} of cluster labels for each example. This method can work if the embedding space is structured so that examples whose cluster label set is $\{a,b\}$ lie near the midpoint between those examples whose  label set is $\{a\}$ and those whose label set is $\{b\}$. Based on this approach, we tested both {\bf Fuzzy $c$-Means (FCM)} and {\bf Gaussian Mixture Models (GMM)} as baselines. 

{\bf Oracle singleton clustering}:
To assess how well a \emph{perfect} clustering method would work that can determine the cluster memberships exactly but not infer compositionality, we  include an Oracle Singleton Clustering (OSC) baseline. Note  it is not meaningful when measuring  ARI (since it would be 100\%); hence, we use it only for the CRI.

\subsection{Experiment I: LibriSpeech}
\label{sec:libri}
\begin{figure}
\begin{center}
\includegraphics[width=\textwidth,trim={0 1.36cm 0 0},clip]{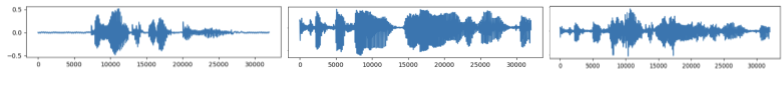}
\end{center}
\caption{Some audio examples from LibriSpeech. From left to right are the waveforms from speaker 1, speaker 2, and  overlapping speech from both speakers.}
\label{fig:librispeech_examples}
\end{figure}
Real-world conversations and meetings often contain moments when multiple people are speaking simultaneously (due to interruptions, sub-group conversations, etc.). Hence, an important few-shot learning problem is to identify the \emph{set} of people speaking at any given moment in time, where the classes (people) at test time usually differ from the classes at training time. 
We thus used the LibriSpeech \cite{panayotov2015librispeech} dataset to explore how well each clustering method
can cluster speech samples into speaker sets and infer the compositional relationships between sets.
LibriSpeech is a corpus of approximately 1000 hours of English audiobook speech from 2484 speakers.
While it contains only individual speakers, we can synthesize simultaneous speech by combining individual tracks, similarly to work by \cite{fujita2019end,hershey2016deep,menne2019analysis}. See
Figure \ref{fig:librispeech_examples} for some examples of the audio waveforms. 

{\bf Embedding model}: We used LibriSpeech to train a
compositional embedding model ($f^\textrm{emb}$ and $g$) for speaker verification using an LSTM neural network 
on top of MFCC audio features (see Appendices for details). Importantly, none of the classes (speakers) that were used for optimizing these networks were used in the clustering experiments. \footnote{For the baseline clustering methods, we also tried training a simpler embedding model   $f^\textrm{emb}$ \emph{without} jointly training $g$ to check whether that gave better performance. However, we found that this actually resulted in worse performance for the baselines, and hence we abandoned it.}

{\bf Procedure}: Our first experiment considers compositionality of degree at most $d=2$. We created datasets $\mathcal{X}$ of size $n\in \{150,750,1500,7500,15000\}$; each dataset contained speech segments from 5 different speakers (picked from 100 speakers which were not seen during training). Some of the segments contained single speakers, and some contained combinations of two speakers. Hence, there were  ${5\choose 1}+{5\choose 2}=15$ different unique speaker sets in total. The test set contains 10 data trials for each $n$ and the validation set has 10 trials when $n=150$ (hyperparameters are picked based on $n=150$ and used for all $n$s).
For each $n$, we compared all three compositional clustering algorithms and all the baselines described above and then compared the resulting CRI  (Section \ref{sec:cri}) and  Adjusted Rand Index (ARI) scores.  For CAP, we used the full inference procedure for $n=150$, and we used CAP$\subset$ with a random subset of $150$ examples for $n>150$. All results of all clustering methods are averaged over 10 trials for each $n$.

To illustrate compositional clustering for $d=3$ (i.e., up to 3 speakers speaking simultaneously), we performed a second experiment using the same composition function $g$ as for $d=2$ (i.e., it does not need to be retrained for different $d$). There are 25 classes in total (5 singletons, 10 2-sets, and 10 3-sets). We adopted the same hyperparameters for each method that were optimized in the previous experiment for $d=2$. Due to the high computational cost, we varied $n$ only up to 2500, and we did not try CAP.

{\bf Hyperparameter optimization}:
CAP, CAP$\subset$, and AP have one hyperparameter, which is the cost $\gamma$ of creating a new singleton cluster.
GCR has two hyperparameters: the first is the number of clusters in the first step of clustering and the second  is the threshold $\tau$ to stop compositional label assignment.
CKM has four hyperparameters: the number of singleton clusters $k$; the number of random initializations; the maximum number of alternations; and the learning rate. 
AC  uses a distance threshold hyperparameter that determines whether to merge two clusters. 
GMM has one hyperparameter to decide the number of components. 
FCM has the number of clusters $c$ and the corresponding threshold $1/c$ on the vector of probabilities that determines when the model infers that an example belongs to a compositional cluster; it also has a temperature  $m$ that can make the estimated class probabilities more or less entropic.
The  hyperparameter sets were  decided separately for each method, based on pilot exploration, to give each method a good
chance of succeeding. The hyperparameter values were then optimized  so as to maximize the average (over 10 trials) CRI. Experiments were conducted using the Python code in the Github repository; the sklearn implementations of
\begin{tt}AgglomerativeClustering\end{tt} and  \begin{tt}AffinityPropagation\end{tt}; and the \begin{tt}SciKit-Fuzzy\end{tt} of FCM.

\begin{table}
\setlength\tabcolsep{3.5pt}
\begin{center}
\begin{tabular}[b]{l|c|c|c|c|c}
\multicolumn{6}{c}{\bf LibriSpeech Results (2 Spkrs, CRI\%)} \\ \hline
$n$  & 150 & 750 & 1500 & 7500 & 15000\\ \hline
GCR & 94.6 (1.1)  & 95.8 (0.7) & 96.0 (0.8)& {\bf96.5 (0.5)} &{\bf96.6 (0.5)}\\
CAP & 94.8 (1.4) &93.8 (1.4)& 93.7 (1.5)& 93.4 (1.3)&92.6 (2.1)\\
CKM & {\bf95.7 (0.8)} & {\bf96.3 (0.6)} &{\bf96.1 (0.6)} &96.2 (0.4) &96.1 (0.5)\\ \hline
AP & 87.9 (0.6) & 87.0 (0.2) & 86.1 (0.1)&85.0 (0.1)&84.8 (0.0)\\
AC & 88.4 (0.3) & 86.4 (0.2) & 85.6 (0.1)&84.7 (0.0)&84.6 (0.0)\\
FCM & 88.0 (0.5) &88.3 (0.5) & 88.1 (0.4) & 88.3 (0.4)&88.4 (0.4)\\
GMM & 87.8 (0.4) & 88.9 (0.4) & 88.7 (0.4) & 88.8 (0.3) & 88.6 (0.5)\\
{\em OSC} & {\em 91.1 (0.0)} & {\em 91.1 (0.0)} & {\em 91.1 (0.0)} & {\em 91.1 (0.0)} & {\em 91.1 (0.0)}
\end{tabular}
\\
\begin{tabular}[b]{l|c|c|c|c|c}
\multicolumn{6}{c}{\bf LibriSpeech Results (2 Spkrs, ARI\%)} \\ \hline
$n$  & 150 & 750 & 1500 & 7500 & 15000\\ \hline
GCR &  77.2 (4.0) & {\bf82.1 (3.4)} & {\bf82.3 (3.7)} &{\bf85.1 (2.9)}&{\bf85.8 (2.6)}\\
CAP & 76.5 (5.5)  & 74.2 (4.9) &74.1 (5.4)&73.2 (4.8)&72.5 (6.2)\\
CKM & {\bf78.5 (3.7)}&81.8 (2.8)&81.4 (2.9)&83.4 (2.2)&83.1 (2.2)\\ \hline
AP &  72.5 (5.1) & 54.9 (2.9) & 39.6 (2.2)&15.3 (0.9)&9.2 (0.6)\\
AC &  74.2 (3.7)  & 43.1 (2.5) & 28.6 (1.8)&8.4 (0.6)&4.7 (0.3)\\
FCM &  70.0 (4.3)  & 70.5 (3.8) & 69.0 (3.8)&70.7 (4.0)&71.8 (3.7)\\
GMM & 70.5 (4.0)&80.3 (3.6)&78.6 (3.6)&79.0 (3.2)&78.8 (3.5)\\
\end{tabular}
\end{center}
\caption{Results of LibriSpeech experiments (2 speakers) in CRI\%/ARI\% along with s.e.}
\label{tbl:librispeech_results}
\end{table}

\begin{table}
\setlength\tabcolsep{3.5pt}
\begin{center}
\begin{tabular}[b]{l|c|c|c}
\multicolumn{4}{c}{\bf LibriSpeech Results (3 Spkrs, CRI\%)} \\ \hline
$n$  & 250 & 1250 & 2500\\ \hline
GCR & 87.4 (0.5)&88.6 (1.1)&89.6( 1.3)\\
CKM & {\bf95.1 (0.4)}&{\bf93.1 (0.7)}&{\bf94.7 (0.4)}\\ \hline
AP & 80.8 (0.2)&84.2 (0.4)&84.9 (0.2)\\
AC & 84.8 (0.1)&84.0 (0.1)&83.7 (0.0)\\
FCM & 84.0 (0.1)&83.8 (0.3)&84.4 (0.2)\\
GMM & 84.3 (0.2)&84.5 (0.4)&85.0 (0.2)\\
\emph{OSC} & \emph{87.1 (0.0)}&\emph{87.2 (0.0)}&\emph{87.2 (0.0)} \\
\end{tabular}

\begin{tabular}[b]{l|c|c|c}
\multicolumn{4}{c}{\bf LibriSpeech Results (3 Spkrs, ARI\%)} \\ \hline
$n$  & 250 & 1250 & 2500\\ \hline
GCR &  55.7 (2.0)&59.0 (4.6)&69.3 (3.0)\\ 
CKM & {\bf72.6 (2.0)}&{\bf64.1 (3.1)}&{\bf71.6 (1.8)}\\ \hline
AP &  32.8 (1.5)&57.9 (4.6)&62.7 (2.6)\\ 
AC &  54.2 (2.7)&32.9 (2.1)&22.6 (0.7)\\ 
FCM &  51.2 (2.0)&50.4 (3.6)&57.0 (2.4)\\ 
GMM & 53.3 (3.3)&60.4 (4.9)&67.8 (3.1) 
\end{tabular}
\end{center}
\caption{Results of LibriSpeech experiments (3 speakers) in CRI\%/ARI\% along with s.e.}
\label{tbl:librispeech_results3}
\end{table}

{\bf Results}: The mean CRI\% and ARI\% (along with standard error) for $d=2$ are shown in Table~\ref{tbl:librispeech_results}. For all values of $n$, all the compositional clustering algorithms (the first three lines of the table) worked better than all of the standard clustering methods in terms of CRI; the differences (as assessed with matched-pair t-tests between methods over the 10 test trials, at the 0.05 significance level) were all stat.~sig. 
GCR and CKM usually gave the highest accuracy. CAP comes in second place for $n=150$ but as $n$ increases, its accuracy decreases; this is likely because, for the larger $n$ values, CAP$\subset$ sees a relatively smaller fraction of the total dataset as $n$ increases. All the proposed methods  outperformed the Oracle Singleton Clustering baseline, suggesting that they can both form coherent clusters and correctly infer the compositional relationships between them. Among the traditional  methods, either FCM or GMM usually performed best: while it does have some ability to infer compositionality via the probability vector assigned to each example, it does not  use  the embedding model's composition function $g$; hence, it must rely  on compositional clusters lying close to their constituent singleton clusters in the embedding space, which does not always happen in practice.

In terms of ARI -- which can measure the purity of inferred clusters but not the accuracy of the inferred compositional relationships -- both GCR and CKM always outperformed the best standard clustering method (though the differences were not always stat.~sig.). This suggests that the $g$ function enabled the compositional methods to obtain purer clusters. CAP outperformed all the standard clustering methods for $n=150$ but not for larger $n$ (when CAP$\subset$ was used). Note that some ARI scores, particularly for the AC and AP methods, were very low for larger $n$; this is likely because the hyperparameters for all methods were optimized for CRI, not ARI, and because ARI adjusts for the accuracy obtained by just guessing. Manual inspection of  the results suggests that some methods (e.g., AC for $n=15000$) incorrectly deduced a very large number of clusters, which was heavily penalized by the ARI metric.

For $d=3$  (Table \ref{tbl:librispeech_results3}), the trends were mostly similar to $d=2$. The accuracy differences  with the best-performing standard clustering method were stat.~sig.~for both GCR and CKM for CRI for all $n$; for ARI, the significance tests were mixed. 

\begin{figure*}
\begin{center}
\includegraphics[width=\textwidth]{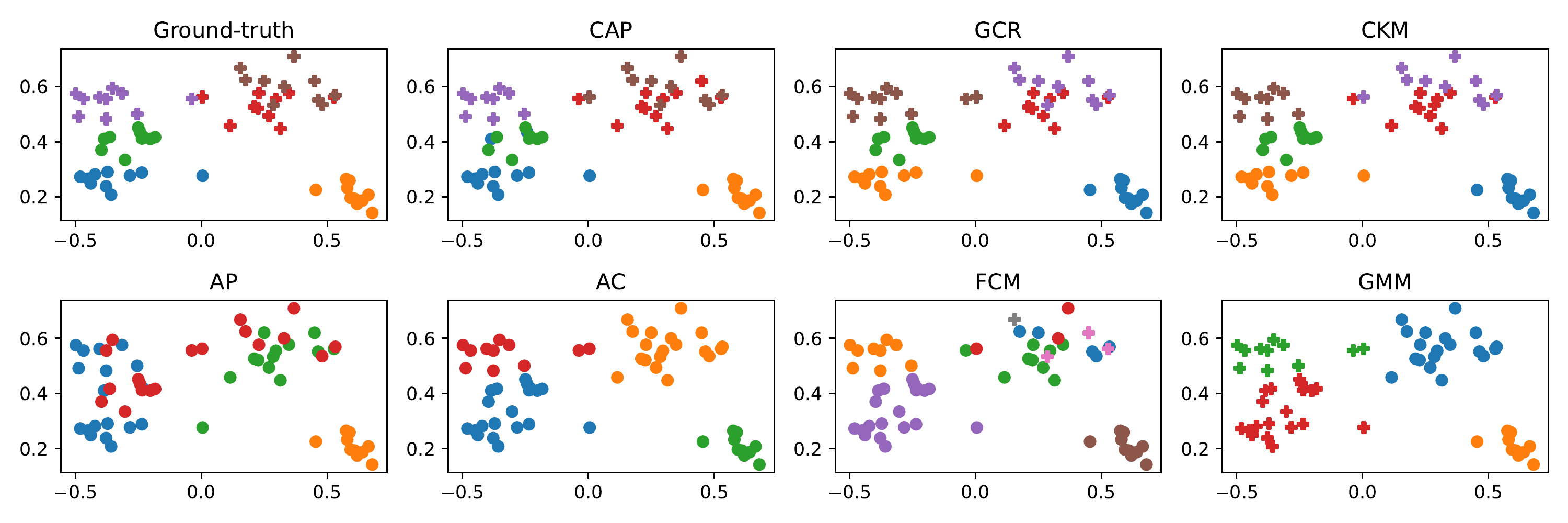}
\end{center}
\caption{Some LibriSpeech clusters (for $d=2$), according to different methods. Circles belong to singleton clusters; plusses belong to compositional clusters. Best viewed in color.}
\label{fig:librispeech_clusters}
\end{figure*}

{\bf Inferred Clusters}:
Figure \ref{fig:librispeech_clusters} shows some clustering results of the 4 different methods; in each plot (generated
using PCA applied to $\mathcal{X}$), circles and plusses represent examples
from singleton and compositional clusters, respectively.
(To avoid clutter, we show just 3 singleton clusters and their compositions, and the inferred relationships of which clusters are composed to yield other clusters are not shown.)
All three compositional clustering methods are largely successful in inferring both the clusters and
their compositionality. AP and AC 
can approximately infer the clusters but sometimes lump groups of examples together
that actually come from distinct clusters. FCM  and GMM do manage to infer some compositionality correctly, but not as well
as the compositional methods.

\subsection{Experiment II: OmniGlot}
\begin{figure}
\begin{center}
	\begin{tabular}[b]{cccccc}
		\includegraphics[width=.3in]{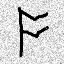} &
		\includegraphics[width=.3in]{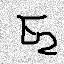} &
		\includegraphics[width=.3in]{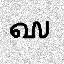} &
		\includegraphics[width=.3in]{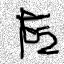} &
		\includegraphics[width=.3in]{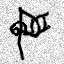} &
		\includegraphics[width=.3in]{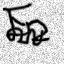} \\
		\includegraphics[width=.3in]{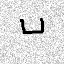} &
		\includegraphics[width=.3in]{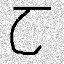} &
		\includegraphics[width=.3in]{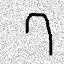} &
		\includegraphics[width=.3in]{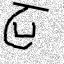} &
		\includegraphics[width=.3in]{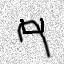} &
		\includegraphics[width=.3in]{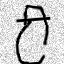} \\
	\end{tabular}
\end{center}
\caption{Some representative examples from OmniGlot images: the first three columns show examples from singleton clusters, whereas the latter show images from compositional clusters.}
\label{fig:omniglot_examples}
\end{figure}
Here we considered a multi-object image recognition problem using the OmniGlot \cite{lake2015human} dataset.
OmniGlot  contains images of handwritten symbols from many languages. It has 1623 different handwritten characters from 50 different alphabets. We can synthesize images with multiple symbols by element-wise superposition;
Figure \ref{fig:omniglot_examples} shows 4 groups of examples used in the experiment, where each group contains images with cluster labels $\{1\},\{2\},\{3\},\{1,2\},\{1,3\},\{2,3\}$. Because of the intertangling of the different symbols in each compositional image, the recognition problem is quite challenging.
We  created 10 trials for each $n$ in the test set and 10 trials when $n=150$ in the validation set.

{\bf Embedding model}:
In our experiment, all characters are augmented with random scaling and shifting, and random Gaussian noise is added to the background. We pre-trained a compositional embedding model $f^\textrm{emb}$ using
a ResNet18 \cite{he2016deep} network; the composition function $g$ is the same as for LibriSpeech (see Appendices).

{\bf Procedures} are analogous to the experiments on LibriSpeech (see Section \ref{sec:libri}).
\begin{table}
\setlength\tabcolsep{3.5pt}
\begin{center}
\begin{tabular}[b]{l|c|c|c|c|c}
\multicolumn{6}{c}{\bf OmniGlot Results (CRI\%)} \\ \hline
$n$ & 150 & 750 & 1500 & 7500 & 15000 \\ \hline
GCR &  {\bf94.9 (0.4)}&95.9 (0.3)&96.0 (0.2)& 96.3 (0.3)&96.3 (0.3)\\ 
CAP &93.3 (0.4)&92.8 (0.5)&92.6 (0.5)&92.7 (0.5)& 92.8 (0.2)\\ 
CKM & 94.3 (1.2)&{\bf97.1 (0.4)}&{\bf96.7 (0.5)}&{\bf96.9 (0.4)}&{\bf96.9 (0.5)}\\ \hline
AP &  87.9 (0.1)& 86.3 (0.1)&85.6 (0.1)&84.8 (0.0)& 84.7 (0.0)\\ 
AC &  87.9 (0.1)&85.7 (0.1)&85.2 (0.0)&84.7 (0.0) & 84.6 (0.0)\\ 
FCM &  88.1 (0.1)&88.0 (0.1)&87.9 (0.2)&88.0 (0.1)&87.9 (0.3)\\ 
GMM &85.7 (0.4)&86.4 (0.5)&86.0 (0.9)&87.9 (0.3)&87.1 (0.7) \\
{\em OSC} & {\em 91.1 (0.0)} & {\em 91.1 (0.0)} & {\em 91.1 (0.0)} & {\em 91.1 (0.0)} & {\em 91.1 (0.0)}
\end{tabular}

\begin{tabular}[b]{l|c|c|c|c|c}
\multicolumn{6}{c}{\bf OmniGlot Results (ARI\%)} \\ \hline
$n$ & 150 & 750 & 1500 & 7500 & 15000 \\ \hline
GCR &  76.0 (1.3)&81.0 (1.2)&81.4 (1.0)&84.6 (0.9)&84.9 (0.9)\\ 
CAP &63.1 (2.2)&64.9 (2.5)&64.3 (2.2)&64.7 (1.9)&66.1 (0.9)\\ 
CKM & {\bf77.7 (4.1)}&{\bf86.1 (1.3)}&{\bf85.5 (1.5)}&{\bf85.7 (1.3)}&{\bf85.8 (1.6)}\\ \hline
AP &  69.8 (1.3)& 43.8 (1.7)&29.3 (1.3)&9.3 (0.4)& 6.3 (0.1)\\ 
AC &  65.4 (1.8)&29.1 (1.6)&18.1 (0.9)&5.8 (0.2)& 4.4 (0.1)\\ 
FCM &  69.7 (1.1)&67.1 (1.0)&66.6 (1.5)&67.1 (1.4)&67.7 (1.5)\\
GMM & 55.1 (2.2)&63.1 (2.7)&65.5 (4.4)&75.0 (2.3)&70.0 (3.3)\\
\end{tabular}
\end{center}
\caption{Results of OmniGlot experiments in CRI\%/ARI\% along with s.e.}
\label{tbl:omniglot_results}
\end{table}

{\bf Results}  are in Table~\ref{tbl:omniglot_results}. CKM  gave the highest accuracy for almost all values of $n$. With CKM (both on LibriSpeech and OmniGlot), we found that trying different random initializations of the singleton cluster centroids, and then choosing the final clustering based on the sum of squared distances after training, was important to get good performance. Nonetheless, CKM's accuracy was not just due to randomly ``guessing'' which of the clusters were singleton clusters -- the number of random seeds in our experiments (we used 100) was far smaller than the total number of possible choices of 5 singleton clusters out of 15 total clusters (${15 \choose 5} = 3003$), suggesting that CKM uses $g$ and the numerical SSD-minimization procedure to deduce compositional structure.
After CKM, 
GCR was usually second best, followed by CAP (which was sometimes slightly worse than the best standard clustering methods, in terms of ARI). Among the traditional clustering methods, 
FCM usually performed best. The CRI accuracy improvement compared to the best standard clustering method was stat.~sig.~for all the compositional clustering algorithms. For ARI, the results of the t-tests were mixed.

\section{Conclusions}
We presented three new algorithms (CAP, CKM,and GCR) 
that can both cluster data and infer the compositional relationships between clusters. These algorithms  can facilitate data visualization and exploratory data analyses on datasets where the classes have not been previously seen. Our experiments  
on the LibriSpeech and OmniGlot datasets suggest that modeling compositionality
explicitly is useful and enables the proposed methods to identify coherent and distinctive clusters, and also to infer the compositional relationships between them. The proposed methods deliver substantially higher accuracy than can be achieved with standard  methods (e.g., GMM, FCM), even when the latter have the ability to assign examples ``softly'' to multiple clusters. Among CKM, GCR, and CAP, we found that CKM and GCR gave higher accuracy and also scale better with dataset size $n$.

{\bf Limitations of proposed methods}: In practice, training  embedding function
$f^\textrm{emb}$ jointly with
composition
function $g$  can be challenging, especially when the compositional degree $d\geq 3$. On the other hand, as few-shot learning is an active research  field, more powerful embedding approaches could arise that make this challenge less severe.

{\bf Future work}: We anticipate that, as multi-label few-shot learning research continues to grow, there will be increasing interest for methods to cluster data from unseen classes. Research on more accurate compositional embedding models, especially to the extent that the composition function $g$ can be increased in accuracy, will likely lead to accuracy improvements in compositional clustering methods. One possible downstream application of our work is a simplified pipeline for speaker diarization: instead of separate algorithms to detect overlapping speech, separate speech segments into long vs.~short turns, and then cluster the utterances \cite{bullock2020overlap}, it may be possible to apply a compositional clustering algorithm that can diarize the set of all speech utterances in just one pass.

{\bf Acknowledgement}: This research was supported by the NSF National AI Institute
for Student-AI Teaming (iSAT) under grant DRL \#2019805. The opinions expressed
are those of the authors and do not represent views of the NSF. The research was also supported by NSF awards \#2046505 and \#1822768.

\bibliography{paper}

\begin{thebibliography}{37}
\expandafter\ifx\csname natexlab\endcsname\relax\def\natexlab#1{#1}\fi
\providecommand{\url}[1]{\texttt{#1}}
\providecommand{\href}[2]{#2}
\providecommand{\path}[1]{#1}
\providecommand{\DOIprefix}{doi:}
\providecommand{\ArXivprefix}{arXiv:}
\providecommand{\URLprefix}{URL: }
\providecommand{\Pubmedprefix}{pmid:}
\providecommand{\doi}[1]{\href{http://dx.doi.org/#1}{\path{#1}}}
\providecommand{\Pubmed}[1]{\href{pmid:#1}{\path{#1}}}
\providecommand{\bibinfo}[2]{#2}
\ifx\xfnm\relax \def\xfnm[#1]{\unskip,\space#1}\fi
\bibitem[{Alfassy et~al.(2019)Alfassy, Karlinsky, Aides, Shtok, Harary, Feris,
  Giryes and Bronstein}]{alfassy2019laso}
\bibinfo{author}{Alfassy, A.}, \bibinfo{author}{Karlinsky, L.},
  \bibinfo{author}{Aides, A.}, \bibinfo{author}{Shtok, J.},
  \bibinfo{author}{Harary, S.}, \bibinfo{author}{Feris, R.},
  \bibinfo{author}{Giryes, R.}, \bibinfo{author}{Bronstein, A.M.},
  \bibinfo{year}{2019}.
\newblock \bibinfo{title}{Laso: Label-set operations networks for multi-label
  few-shot learning}, in: \bibinfo{booktitle}{Proceedings of the IEEE/CVF
  conference on computer vision and pattern recognition}, pp.
  \bibinfo{pages}{6548--6557}.
\bibitem[{Allen and Hospedales(2019)}]{allen2019analogies}
\bibinfo{author}{Allen, C.}, \bibinfo{author}{Hospedales, T.},
  \bibinfo{year}{2019}.
\newblock \bibinfo{title}{Analogies explained: Towards understanding word
  embeddings}, in: \bibinfo{booktitle}{International Conference on Machine
  Learning}, \bibinfo{organization}{PMLR}. pp. \bibinfo{pages}{223--231}.
\bibitem[{Bezdek et~al.(1984)Bezdek, Ehrlich and Full}]{bezdek1984fcm}
\bibinfo{author}{Bezdek, J.C.}, \bibinfo{author}{Ehrlich, R.},
  \bibinfo{author}{Full, W.}, \bibinfo{year}{1984}.
\newblock \bibinfo{title}{Fcm: The fuzzy c-means clustering algorithm}.
\newblock \bibinfo{journal}{Computers \& geosciences} \bibinfo{volume}{10},
  \bibinfo{pages}{191--203}.
\bibitem[{Bickel and Scheffer(2004)}]{bickel2004multi}
\bibinfo{author}{Bickel, S.}, \bibinfo{author}{Scheffer, T.},
  \bibinfo{year}{2004}.
\newblock \bibinfo{title}{Multi-view clustering.}, in:
  \bibinfo{booktitle}{ICDM}, \bibinfo{organization}{Citeseer}. pp.
  \bibinfo{pages}{19--26}.
\bibitem[{Blei and Jordan(2006)}]{blei2006variational}
\bibinfo{author}{Blei, D.M.}, \bibinfo{author}{Jordan, M.I.},
  \bibinfo{year}{2006}.
\newblock \bibinfo{title}{Variational inference for dirichlet process
  mixtures}.
\newblock \bibinfo{journal}{Bayesian analysis} \bibinfo{volume}{1},
  \bibinfo{pages}{121--143}.
\bibitem[{Bullock et~al.(2020)Bullock, Bredin and
  Garcia-Perera}]{bullock2020overlap}
\bibinfo{author}{Bullock, L.}, \bibinfo{author}{Bredin, H.},
  \bibinfo{author}{Garcia-Perera, L.P.}, \bibinfo{year}{2020}.
\newblock \bibinfo{title}{Overlap-aware diarization: Resegmentation using
  neural end-to-end overlapped speech detection}, in:
  \bibinfo{booktitle}{ICASSP 2020-2020 IEEE International Conference on
  Acoustics, Speech and Signal Processing (ICASSP)},
  \bibinfo{organization}{IEEE}. pp. \bibinfo{pages}{7114--7118}.
\bibitem[{Chen et~al.(2020)Chen, Lin, Chen, Hui and Wu}]{chen2020knowledge}
\bibinfo{author}{Chen, T.}, \bibinfo{author}{Lin, L.}, \bibinfo{author}{Chen,
  R.}, \bibinfo{author}{Hui, X.}, \bibinfo{author}{Wu, H.},
  \bibinfo{year}{2020}.
\newblock \bibinfo{title}{Knowledge-guided multi-label few-shot learning for
  general image recognition}.
\newblock \bibinfo{journal}{IEEE Transactions on Pattern Analysis and Machine
  Intelligence} \bibinfo{volume}{44}, \bibinfo{pages}{1371--1384}.
\bibitem[{Deng et~al.(2019)Deng, Guo, Xue and Zafeiriou}]{deng2019arcface}
\bibinfo{author}{Deng, J.}, \bibinfo{author}{Guo, J.}, \bibinfo{author}{Xue,
  N.}, \bibinfo{author}{Zafeiriou, S.}, \bibinfo{year}{2019}.
\newblock \bibinfo{title}{Arcface: Additive angular margin loss for deep face
  recognition}, in: \bibinfo{booktitle}{Proceedings of the IEEE/CVF conference
  on computer vision and pattern recognition}, pp. \bibinfo{pages}{4690--4699}.
\bibitem[{Franklin and Frank(2018)}]{franklin2018compositional}
\bibinfo{author}{Franklin, N.T.}, \bibinfo{author}{Frank, M.J.},
  \bibinfo{year}{2018}.
\newblock \bibinfo{title}{Compositional clustering in task structure learning}.
\newblock \bibinfo{journal}{PLoS computational biology} \bibinfo{volume}{14},
  \bibinfo{pages}{e1006116}.
\bibitem[{Frey and Dueck(2007)}]{frey2007clustering}
\bibinfo{author}{Frey, B.J.}, \bibinfo{author}{Dueck, D.},
  \bibinfo{year}{2007}.
\newblock \bibinfo{title}{Clustering by passing messages between data points}.
\newblock \bibinfo{journal}{science} \bibinfo{volume}{315},
  \bibinfo{pages}{972--976}.
\bibitem[{Fu et~al.(2020)Fu, Lin, Vasilakos and Wang}]{fu2020overview}
\bibinfo{author}{Fu, L.}, \bibinfo{author}{Lin, P.},
  \bibinfo{author}{Vasilakos, A.V.}, \bibinfo{author}{Wang, S.},
  \bibinfo{year}{2020}.
\newblock \bibinfo{title}{An overview of recent multi-view clustering}.
\newblock \bibinfo{journal}{Neurocomputing} \bibinfo{volume}{402},
  \bibinfo{pages}{148--161}.
\bibitem[{Fujita et~al.(2019)Fujita, Kanda, Horiguchi, Nagamatsu and
  Watanabe}]{fujita2019end}
\bibinfo{author}{Fujita, Y.}, \bibinfo{author}{Kanda, N.},
  \bibinfo{author}{Horiguchi, S.}, \bibinfo{author}{Nagamatsu, K.},
  \bibinfo{author}{Watanabe, S.}, \bibinfo{year}{2019}.
\newblock \bibinfo{title}{End-to-end neural speaker diarization with
  permutation-free objectives}.
\newblock \bibinfo{journal}{arXiv preprint arXiv:1909.05952} .
\bibitem[{Hariharan and Girshick(2017)}]{hariharan2017low}
\bibinfo{author}{Hariharan, B.}, \bibinfo{author}{Girshick, R.},
  \bibinfo{year}{2017}.
\newblock \bibinfo{title}{Low-shot visual recognition by shrinking and
  hallucinating features}, in: \bibinfo{booktitle}{Proceedings of the IEEE
  international conference on computer vision}, pp.
  \bibinfo{pages}{3018--3027}.
\bibitem[{He et~al.(2016)He, Zhang, Ren and Sun}]{he2016deep}
\bibinfo{author}{He, K.}, \bibinfo{author}{Zhang, X.}, \bibinfo{author}{Ren,
  S.}, \bibinfo{author}{Sun, J.}, \bibinfo{year}{2016}.
\newblock \bibinfo{title}{Deep residual learning for image recognition}, in:
  \bibinfo{booktitle}{Proceedings of the IEEE conference on computer vision and
  pattern recognition}, pp. \bibinfo{pages}{770--778}.
\bibitem[{Hershey et~al.(2016)Hershey, Chen, Le~Roux and
  Watanabe}]{hershey2016deep}
\bibinfo{author}{Hershey, J.R.}, \bibinfo{author}{Chen, Z.},
  \bibinfo{author}{Le~Roux, J.}, \bibinfo{author}{Watanabe, S.},
  \bibinfo{year}{2016}.
\newblock \bibinfo{title}{Deep clustering: Discriminative embeddings for
  segmentation and separation}, in: \bibinfo{booktitle}{2016 IEEE International
  Conference on Acoustics, Speech and Signal Processing (ICASSP)},
  \bibinfo{organization}{IEEE}. pp. \bibinfo{pages}{31--35}.
\bibitem[{Huynh and Elhamifar(2021)}]{huynh2021interaction}
\bibinfo{author}{Huynh, D.}, \bibinfo{author}{Elhamifar, E.},
  \bibinfo{year}{2021}.
\newblock \bibinfo{title}{Interaction compass: Multi-label zero-shot learning
  of human-object interactions via spatial relations}, in:
  \bibinfo{booktitle}{Proceedings of the IEEE/CVF International Conference on
  Computer Vision}, pp. \bibinfo{pages}{8472--8483}.
\bibitem[{Lake et~al.(2015)Lake, Salakhutdinov and Tenenbaum}]{lake2015human}
\bibinfo{author}{Lake, B.M.}, \bibinfo{author}{Salakhutdinov, R.},
  \bibinfo{author}{Tenenbaum, J.B.}, \bibinfo{year}{2015}.
\newblock \bibinfo{title}{Human-level concept learning through probabilistic
  program induction}.
\newblock \bibinfo{journal}{Science} \bibinfo{volume}{350},
  \bibinfo{pages}{1332--1338}.
\bibitem[{Lee et~al.(2018)Lee, Fang, Yeh and Wang}]{lee2018multi}
\bibinfo{author}{Lee, C.W.}, \bibinfo{author}{Fang, W.}, \bibinfo{author}{Yeh,
  C.K.}, \bibinfo{author}{Wang, Y.C.F.}, \bibinfo{year}{2018}.
\newblock \bibinfo{title}{Multi-label zero-shot learning with structured
  knowledge graphs}, in: \bibinfo{booktitle}{Proceedings of the IEEE conference
  on computer vision and pattern recognition}, pp. \bibinfo{pages}{1576--1585}.
\bibitem[{Li et~al.(2021)Li, Mozer and Whitehill}]{li2021compositional}
\bibinfo{author}{Li, Z.}, \bibinfo{author}{Mozer, M.},
  \bibinfo{author}{Whitehill, J.}, \bibinfo{year}{2021}.
\newblock \bibinfo{title}{Compositional embeddings for multi-label one-shot
  learning}, in: \bibinfo{booktitle}{Proceedings of the IEEE/CVF Winter
  Conference on Applications of Computer Vision}, pp.
  \bibinfo{pages}{296--304}.
\bibitem[{Li and Whitehill(2021)}]{li2021compositionalEmb}
\bibinfo{author}{Li, Z.}, \bibinfo{author}{Whitehill, J.},
  \bibinfo{year}{2021}.
\newblock \bibinfo{title}{Compositional embedding models for speaker
  identification and diarization with simultaneous speech from 2+ speakers},
  in: \bibinfo{booktitle}{ICASSP 2021-2021 IEEE International Conference on
  Acoustics, Speech and Signal Processing (ICASSP)},
  \bibinfo{organization}{IEEE}. pp. \bibinfo{pages}{7163--7167}.
\bibitem[{Menne et~al.(2019)Menne, Sklyar, Schl{\"u}ter and
  Ney}]{menne2019analysis}
\bibinfo{author}{Menne, T.}, \bibinfo{author}{Sklyar, I.},
  \bibinfo{author}{Schl{\"u}ter, R.}, \bibinfo{author}{Ney, H.},
  \bibinfo{year}{2019}.
\newblock \bibinfo{title}{Analysis of deep clustering as preprocessing for
  automatic speech recognition of sparsely overlapping speech}.
\newblock \bibinfo{journal}{arXiv preprint arXiv:1905.03500} .
\bibitem[{Miller(1995)}]{miller1995wordnet}
\bibinfo{author}{Miller, G.A.}, \bibinfo{year}{1995}.
\newblock \bibinfo{title}{Wordnet: a lexical database for english}.
\newblock \bibinfo{journal}{Communications of the ACM} \bibinfo{volume}{38},
  \bibinfo{pages}{39--41}.
\bibitem[{Narayan et~al.(2021)Narayan, Gupta, Khan, Khan, Shao and
  Shah}]{narayan2021discriminative}
\bibinfo{author}{Narayan, S.}, \bibinfo{author}{Gupta, A.},
  \bibinfo{author}{Khan, S.}, \bibinfo{author}{Khan, F.S.},
  \bibinfo{author}{Shao, L.}, \bibinfo{author}{Shah, M.}, \bibinfo{year}{2021}.
\newblock \bibinfo{title}{Discriminative region-based multi-label zero-shot
  learning}, in: \bibinfo{booktitle}{Proceedings of the IEEE/CVF International
  Conference on Computer Vision}, pp. \bibinfo{pages}{8731--8740}.
\bibitem[{Pal and Heumann(2022)}]{pal2022clustering}
\bibinfo{author}{Pal, S.}, \bibinfo{author}{Heumann, C.}, \bibinfo{year}{2022}.
\newblock \bibinfo{title}{Clustering compositional data using dirichlet mixture
  model}.
\newblock \bibinfo{journal}{Plos one} \bibinfo{volume}{17},
  \bibinfo{pages}{e0268438}.
\bibitem[{Panayotov et~al.(2015)Panayotov, Chen, Povey and
  Khudanpur}]{panayotov2015librispeech}
\bibinfo{author}{Panayotov, V.}, \bibinfo{author}{Chen, G.},
  \bibinfo{author}{Povey, D.}, \bibinfo{author}{Khudanpur, S.},
  \bibinfo{year}{2015}.
\newblock \bibinfo{title}{Librispeech: an asr corpus based on public domain
  audio books}, in: \bibinfo{booktitle}{2015 IEEE international conference on
  acoustics, speech and signal processing (ICASSP)},
  \bibinfo{organization}{IEEE}. pp. \bibinfo{pages}{5206--5210}.
\bibitem[{Pennington et~al.(2014)Pennington, Socher and
  Manning}]{pennington2014glove}
\bibinfo{author}{Pennington, J.}, \bibinfo{author}{Socher, R.},
  \bibinfo{author}{Manning, C.D.}, \bibinfo{year}{2014}.
\newblock \bibinfo{title}{Glove: Global vectors for word representation}, in:
  \bibinfo{booktitle}{Proceedings of the 2014 conference on empirical methods
  in natural language processing (EMNLP)}, pp. \bibinfo{pages}{1532--1543}.
\bibitem[{Rand(1971)}]{rand1971objective}
\bibinfo{author}{Rand, W.M.}, \bibinfo{year}{1971}.
\newblock \bibinfo{title}{Objective criteria for the evaluation of clustering
  methods}.
\newblock \bibinfo{journal}{Journal of the American Statistical association}
  \bibinfo{volume}{66}, \bibinfo{pages}{846--850}.
\bibitem[{Schroff et~al.(2015)Schroff, Kalenichenko and
  Philbin}]{schroff2015facenet}
\bibinfo{author}{Schroff, F.}, \bibinfo{author}{Kalenichenko, D.},
  \bibinfo{author}{Philbin, J.}, \bibinfo{year}{2015}.
\newblock \bibinfo{title}{Facenet: A unified embedding for face recognition and
  clustering}, in: \bibinfo{booktitle}{Proceedings of the IEEE conference on
  computer vision and pattern recognition}, pp. \bibinfo{pages}{815--823}.
\bibitem[{Snyder et~al.(2018)Snyder, Garcia-Romero, Sell, Povey and
  Khudanpur}]{snyder2018x}
\bibinfo{author}{Snyder, D.}, \bibinfo{author}{Garcia-Romero, D.},
  \bibinfo{author}{Sell, G.}, \bibinfo{author}{Povey, D.},
  \bibinfo{author}{Khudanpur, S.}, \bibinfo{year}{2018}.
\newblock \bibinfo{title}{X-vectors: Robust dnn embeddings for speaker
  recognition}, in: \bibinfo{booktitle}{2018 IEEE International Conference on
  Acoustics, Speech and Signal Processing (ICASSP)},
  \bibinfo{organization}{IEEE}. pp. \bibinfo{pages}{5329--5333}.
\bibitem[{Song et~al.(2021)Song, Tan, Zhao and Yang}]{song2021deep}
\bibinfo{author}{Song, G.}, \bibinfo{author}{Tan, X.}, \bibinfo{author}{Zhao,
  J.}, \bibinfo{author}{Yang, M.}, \bibinfo{year}{2021}.
\newblock \bibinfo{title}{Deep robust multilevel semantic hashing for
  multi-label cross-modal retrieval}.
\newblock \bibinfo{journal}{Pattern Recognition} \bibinfo{volume}{120},
  \bibinfo{pages}{108084}.
\bibitem[{Tibshirani et~al.(2001)Tibshirani, Walther and
  Hastie}]{tibshirani2001estimating}
\bibinfo{author}{Tibshirani, R.}, \bibinfo{author}{Walther, G.},
  \bibinfo{author}{Hastie, T.}, \bibinfo{year}{2001}.
\newblock \bibinfo{title}{Estimating the number of clusters in a data set via
  the gap statistic}.
\newblock \bibinfo{journal}{Journal of the Royal Statistical Society: Series B
  (Statistical Methodology)} \bibinfo{volume}{63}, \bibinfo{pages}{411--423}.
\bibitem[{Ward~Jr(1963)}]{ward1963hierarchical}
\bibinfo{author}{Ward~Jr, J.H.}, \bibinfo{year}{1963}.
\newblock \bibinfo{title}{Hierarchical grouping to optimize an objective
  function}.
\newblock \bibinfo{journal}{Journal of the American statistical association}
  \bibinfo{volume}{58}, \bibinfo{pages}{236--244}.
\bibitem[{Weiss and Freeman(2001)}]{weiss2001optimality}
\bibinfo{author}{Weiss, Y.}, \bibinfo{author}{Freeman, W.T.},
  \bibinfo{year}{2001}.
\newblock \bibinfo{title}{On the optimality of solutions of the max-product
  belief-propagation algorithm in arbitrary graphs}.
\newblock \bibinfo{journal}{IEEE Transactions on Information Theory}
  \bibinfo{volume}{47}, \bibinfo{pages}{736--744}.
\bibitem[{Yang and Wang(2018)}]{yang2018multi}
\bibinfo{author}{Yang, Y.}, \bibinfo{author}{Wang, H.}, \bibinfo{year}{2018}.
\newblock \bibinfo{title}{Multi-view clustering: A survey}.
\newblock \bibinfo{journal}{Big Data Mining and Analytics} \bibinfo{volume}{1},
  \bibinfo{pages}{83--107}.
\bibitem[{Yin et~al.(2018)Yin, Bredin and Barras}]{yin2018neural}
\bibinfo{author}{Yin, R.}, \bibinfo{author}{Bredin, H.},
  \bibinfo{author}{Barras, C.}, \bibinfo{year}{2018}.
\newblock \bibinfo{title}{Neural speech turn segmentation and affinity
  propagation for speaker diarization}, in: \bibinfo{booktitle}{Annual
  Conference of the International Speech Communication Association}.
\bibitem[{Zelenak et~al.(2012)Zelenak, Segura, Luque and
  Hernando}]{zelenak2012simultaneous}
\bibinfo{author}{Zelenak, M.}, \bibinfo{author}{Segura, C.},
  \bibinfo{author}{Luque, J.}, \bibinfo{author}{Hernando, J.},
  \bibinfo{year}{2012}.
\newblock \bibinfo{title}{Simultaneous speech detection with spatial features
  for speaker diarization}.
\newblock \bibinfo{journal}{IEEE Transactions on Audio, Speech, and Language
  Processing} \bibinfo{volume}{20}, \bibinfo{pages}{436--446}.
\bibitem[{Zhou et~al.(2021)Zhou, Huang, Liu and Yang}]{zhou2021multi}
\bibinfo{author}{Zhou, F.}, \bibinfo{author}{Huang, S.}, \bibinfo{author}{Liu,
  B.}, \bibinfo{author}{Yang, D.}, \bibinfo{year}{2021}.
\newblock \bibinfo{title}{Multi-label image classification via category
  prototype compositional learning}.
\newblock \bibinfo{journal}{IEEE Transactions on Circuits and Systems for Video
  Technology} \bibinfo{volume}{32}, \bibinfo{pages}{4513--4525}.

\end{thebibliography}
\bibliographystyle{elsarticle-harv}

\newpage
\appendix



\section{Loopy Belief Propagation for Standard Affinity Propagation}
When applying the max-product algorithm 
to the factor graph in standard Affinity Propagation,
a sequence of ``messages'' (functions $\alpha, \rho: [n]\times[n]\times \mathcal{C} \rightarrow \mathbb{R}_{\leq 0} \cup \{ -\infty \}$) is passed back and forth between the variable and factor nodes.
Each variable $i$ sends a message $\R{i}{k}{c_i}$ to constraint $k$, and
each constraint $k$ sends a message $\A{i}{k}{c_i}$ to variable $i$, about the likelihood of each possible value of $c_i$.
The max-sum algorithm (and the related max-product algorithm for factor graphs) 
dictates that $\R{i}{k}{c_i}$ equals the sum of messages over $c_i$'s neighbors \emph{except} $\delta_k$ (i.e.,
$\{ \delta_{k'} \}_{k' \ne k}$):
\begin{equation}
\R{i}{k}{c_i} = S(i,c_i) + \sum_{k' \ne k} \A{i}{k'}{c_i}
\end{equation}
Also, for MAP estimation, $\A{i}{k}{c_i}$ equals the maximum possible sum of the messages from all of $\delta_k$'s neighbors \emph{except} $i$ (i.e., $\{ c_{i'} \}_{i' \ne i}$), plus the value of $\delta_k$ itself:
\begin{equation}
\A{i}{k}{c_i} = \max_{\{ c_{i'} \}_{i' \ne i}} \left[
\delta_k(c_1,\ldots,c_n) + \sum_{i'\ne i} \R{i'}{k}{c_{i'}} \right]
\end{equation}

\section{Derivation of Algorithm 1 (CAP) and Proof of Theorem 1}
Here we derive Algorithm 1 from the definitions of $\alpha$ and $\rho$  to optimize the Compositional Affinity Propagation model. We also prove the time cost in Theorem 1.
Recall the definitions of $\alpha$, $\rho$, and $\delta$:
\begin{eqnarray}
\R{i}{k}{c_i} &=& S(i,c_i) + \sum_{k' \ne k} \A{i}{k'}{c_i} \label{eqn:rho} \\
\A{i}{k}{c_i} &=& \max_{\{ c_{i'} \}_{i' \ne i}} \left[
\delta_k(c_1,\ldots,c_n) + \sum_{i'\ne i} \R{i'}{k}{c_{i'}} \right] \label{eqn:alpha} \\
\delta_k(c_1,\ldots,c_n)&=& \left\{ \begin{array}{cl}- \infty & \textrm{if}\ \exists i: (c_i\ni k) \wedge (c_k\ne \{k\}) \\
                                                     0 & \textrm{otherwise}
                                   \end{array} \right.
\end{eqnarray}
Each message $\A{i}{k}{c_i}$ computes, for a given value of $c_i$ for variable $i$, an (unnormalized) log-likelihood of the
\emph{best possible configuration} of the assignments of all the \emph{other} variables $\{ c_{i' \ne i} \ \}$,
given that constraint $k$ is satisfied (i.e., $\delta_k$ is finite).
There are four cases in which this occurs; they mirror those in standard Affinity Propagation but differ slightly.
For each case, the $\delta$ term in the RHS of Eqn.~\ref{eqn:alpha} vanishes; the only remaining terms are the sum of the $\rho$'s.  Also, since each summand in Eqn.~\ref{eqn:alpha} depends on just a single unique $c_{i'}$, the max of the sum becomes the sum of the max. Cases:
\begin{enumerate}
\item $i=k$, $c_i=\{k\}$: Since in this case example $i=k$ designates itself as an exemplar, 
then the constraint $\delta_k$ is immediately satisfied.
Moreover, any of the other examples $i'\ne i$ is free to choose (or not choose) example $k$ as an exemplar, and
therefore we can take the maximum over \emph{any} possible value for each $c_{i'}$.
Hence, \begin{eqnarray*}
\A{i}{k}{c_i} &=& 
\max_{\{ c_{i'} \}_{i' \ne k}} \left[ 0 + \sum_{i'\ne k} \R{i'}{k}{c_{i'}} \right]
=\sum_{i' \ne k} \max_{c_{i'}} \R{i'}{k}{c_{i'}}\end{eqnarray*}
\item $i=k, c_i\not \ni k$: Since example $i=k$ does not designate itself as an exemplar, then none of the other examples
$i'\ne i$ may choose $k$ as its exemplar. Hence, $\A{i}{k}{c_i} = \sum_{i' \ne k} \max_{c_{i'} \not \ni k} \R{i'}{k}{c_{i'}}$.
\item $i\ne k, c_i\ni k$: Since example $i$ designates its exemplar either to be or to include example $k$, then $\alpha$ is finite only if $c_k=\{k\}$, and each remaining example $i'\not\in \{i,k\}$ is free to designate any example as its exemplar. Hence,
	$\A{i}{k}{c_i} = \R{k}{k}{\{k\}} + \sum_{i' \not\in \{i,k\}} \max_{c_{i'}} \R{i'}{k}{c_{i'}}$.
\item $i\ne k,c_k \not \ni k$: Since example $i$ does not designate $k$ as an exemplar, then 
example $k$ can either  be an exemplar or not, and we  take the max over both
possibilities:
\begin{eqnarray*}
\A{i}{k}{c_i}
 &=& \max\left[\max_{c_k \not \ni k} \R{k}{k}{c_k} + \sum_{i' \not \in \{i,k\}} \max_{c_{i'} \not \ni k} \R{i'}{k}{c_{i'}},\quad
\R{k}{k}{\{k\}} + \sum_{i' \not \in \{i,k\}} \max_{c_{i'}} \R{i'}{k}{c_i} \right]
\end{eqnarray*}
\end{enumerate}
Note that $\A{i}{k}{c_i}=-\infty$ if $i=k$, $c_i \ni i$ and $c_i \ne \{ i \}$. However, in practice we can avoid this case
by instead setting  $S(i,c_i)=-\infty$ whenever $c_i \ni i$ and $c_i \ne \{ i \}$.
Given the four cases above, we have the following definition of $\alpha$:
\begin{eqnarray}
\A{i}{k}{c_i} &=& \max_{\{ c_{i'} \}_{i'\ne i}} \left[
  \delta_k(c_1,\ldots, c_n) + \sum_{i' \ne i} \R{i'}{k}{c_{i'}}
\right] \\
&=& \left\{
	\begin{array}{lc}
	\sum_{i' \ne k} \max_{c_{i'}} \R{i'}{k}{c_{i'}} & i=k, c_i=\{k\} \\
	\sum_{i' \ne k} \max_{c_{i'} \not \ni k} \R{i'}{k}{c_{i'}} & i=k, c_i \not \ni k\\
	\R{k}{k}{k} + \sum_{i' \not \in \{i,k\}} \max_{c_{i'}} \R{i'}{k}{c_{i'}} & i \ne k, c_i\ni k\\
	\max\left[\max_{c_k \not \ni k} \R{k}{k}{c_k} + \sum_{i' \not \in \{i,k\}} \max_{c_{i'} \not \ni k} \R{i'}{k}{c_{i'}},\right. &\\
	     \quad \qquad \left. \R{k}{k}{k} + \sum_{i' \not \in \{i, k \}} \max_{c_{i'}} \R{i'}{k}{c_{i'}} \right] & i \ne k, c_k \not \ni k\\
	\end{array}
\right. \label{eqn:alpha4cases}
\end{eqnarray}

In the most naive implementation, evaluating $\alpha$ for \emph{each} tuple $(i,k,c_i)$ would take time
$O(n^2)$ due to the summing over the max; the entire table of $\alpha$ values
would thus take time $O(n^4\times |\mathcal{C}|)$. However, there is massive redundancy that can be avoided:
First, for each tuple $(i,k)$, only two possible values of $\A{i}{k}{c_i}$ exist: one for $c_i \ni k$ 
(i.e., $\A{i}{k}{\phi(k)}$) and one for $c_i \not \ni k$ (i.e., $\A{i}{k}{\overline{\phi}(k)}$). 
Hence, instead of computing $|\mathcal{C}|$ values for each tuple $(i,k)$, we need to compute and store only 2 values.
Second, the expressions $\sum_{i' \ne k} \max_{c_{i'}} \R{i'}{k}{c_{i'}}$ 
and $\sum_{i' \ne k} \max_{c_{i'} \not \ni k} \R{i'}{k}{c_{i'}}$
depend on $k$ but not on $i$; hence, they can be reused for many tuples $(i,k)$. Third:
\begin{eqnarray*}
\sum_{i' \not \in \{i,k\}} \max_{c_{i'}} \R{i'}{k}{c_{i'}} &=& \sum_{i' \ne k} \max_{c_{i'}} \R{i'}{k}{c_{i'}}\ - \max_{c_i} \R{i}{k}{c_i} \\
\sum_{i' \not \in \{i,k\}} \max_{c_{i'} \not \ni k} \R{i'}{k}{c_{i'}} &=& \sum_{i' \ne k} \max_{c_{i'} \not \ni k} \R{i'}{k}{c_{i'}}\ -
\max_{c_i \not \ni k} \R{i}{k}{c_{i}}
\end{eqnarray*}
Hence, after computing each of the terms of the LHS above (just once for each $k$),
we need only to ``adjust'' them for each $i$, in $O(1)$ time, by subtracting
the corresponding term on the RHS.
At the end of all the CAP iterations, we  set
$c_i^\textrm{MAP} = \argmax_{c_i} \left[ \sum_k \A{i}{k}{c_i} + S(i,c_i) \right]$.
Hence, as long as we can update $\alpha$ during each iteration of message passing, then
we never need to know $\rho$ explicitly.

For convenience, define the following functions:
\begin{eqnarray*}
b(i,k) &=& \max_{c_i} \R{i}{k}{c_i} \\
\overline{b}(i,k) &=& \max_{c_i \not \ni k} \R{i}{k}{c_i} \\
e(k) &=& \sum_{i' \ne k} \max_{c_{i'}} \R{i'}{k}{c_{i'}} = \sum_{i' \ne k} b(i',k)\\
\overline{e}(k) &=& \sum_{i' \ne k} \max_{c_{i'} \not \ni k} \R{i'}{k}{c_{i'}} = \sum_{i' \ne k} \overline{b}(i',k)\\
h(k) &=& \R{k}{k}{k}\\
a(i,k)&=&\A{i}{k}{\phi(k)} = \left\{\begin{array}{cc}e(k) & i=k \\ h(k) + e(k)-b(i,k) & i\ne k\end{array}\right. \\
\overline{a}(i,k)&=&\A{i}{k}{\overline{\phi}(k)} = \left\{\begin{array}{cc}
    \overline{e}(k) & i=k  \\
    \max(\overline{b}(k,k) + \overline{e}(k) - \overline{b}(i,k), h(k) + e(k) - b(i,k)) & i\ne k
  \end{array}\right.\\
\end{eqnarray*}
Visual inspection of Equation~\ref{eqn:alpha4cases} confirms that the $a(i,k)$ and $\overline{a}(i,k)$
defined above recover all $2n^2$ degrees of freedom of $\alpha$.
Below we show how we can compute $e$,$\overline{e}$,$b$, $\overline{b}$, and $h$
in a \emph{total} time of $O(dn^{d+1})$ per iteration. First, however, we need to derive the computation
of some intermediate quantities.

%

\subsection{Computing $q(i,c_i) = \sum_{k'} \A{i}{k'}{c_i}\ \forall i,c_i$}
\label{sec:q}
Define $q(i,c_i) = \sum_{k'} \A{i}{k'}{c_i}$. For each $i$, we can compute $q(i,c_i)$ for each $c_i$ 
by splitting the sum over $k'$ into two parts: those $k'$ such that $\phi(k') \ni c_i$ and those $k'$
such that $\overline{\phi}(k') \ni c_i$. We then substitute $\A{i}{k'}{c_i} = \A{i}{k'}{\phi(k')}$
for $k'$ s.t.~$\phi(k') \ni c_i$ (and similarly for $\overline{\phi}(k')$) to yield:
\begin{eqnarray*}
q(i,c_i)
&=& \sum_{k': \phi(k') \ni c_i} \A{i}{k'}{\phi(k')} + \sum_{k': \overline{\phi}(k') \ni c_i} \A{i}{k'}{\overline{\phi}(k')} \\
&=& \sum_{k'} \A{i}{k'}{\overline{\phi}(k')}
  + \sum_{k': \phi(k') \ni c_i} (\A{i}{k'}{\phi(k')} - \A{i}{k'}{\overline{\phi}(k')}) \\
&=& \sum_{k'} \A{i}{k'}{\overline{\phi}(k')}
  + \sum_{k' \in c_i} (\A{i}{k'}{\phi(k')} - \A{i}{k'}{\overline{\phi}(k')}) \\
\end{eqnarray*}
We can define ${q}^*(i)=\sum_{k'} \A{i}{k'}{\overline{\phi}(k')}$. Then we have $q(i,c_i) = {q}^*(i) + \sum_{k' \in c_i} (a(i,k') - \overline{a}(i,k'))$.
The term ${q}^*(i)$ takes time $O(n)$ for each $i$ but is reused for all $c_i$.
The summation on the RHS contains at most $d$ terms
(for a maximum composition size of $d$). Hence, for each $i$, the total computation
(over all $c_i$) is $O(n + |\mathcal{C}|d) = O(n + dn^d) = O(dn^d)$.

\subsection{Efficiently Finding Maxima of Many Subsets}
\label{sec:find_all_maxes}
The next step we need is an efficient method to compute expressions of the forms (a)
$\max_{c_i \in \phi(k)} q(i,c_i)$ and (b)
$\max_{c_i \in \overline{\phi}(k)} q(i,c_i)$ for all $k$, in a 
total time of $O(n^{d+1})$.

Form (a): 
Since each such $c_i$
must contain $k$, then there are only $d-1$ remaining degrees of freedom for each $\phi(k)$; hence,
$|\phi(k)|\leq n^{d-1}$ for each $k$, and directly computing the maximum of $q(i,\cdot)$ over every
$\phi(k)$ takes a total time of $O(n^d)$ (summed over all $k$).

Form (b):
Define $\overline{\phi}^j(k) = \{ c \in \overline{\phi}(k): |c| = j \}$.
Since $\max_{c_i \in \overline{\phi}(k)} q(i,c_i) = \max_{j\in [d]} \max_{c_i \in \overline{\phi}^j(k)} q(i,c_i)$,
we can split the task into subtasks by $j$ and then take the max over all of them.
To compute the max over each $\overline{\phi}^j(k)$, we can iterate over all $n^{j-1}$ tuples
$(t_1,\ldots,t_{j-1})\in [n]^{j-1}$; for each tuple, we can compute in $O(n)$ time the largest and second-largest value of $q(i,\cdot)$ over the set $(t-1,\ldots,t_{j-1},t_j)$ and then ``adjust'' the result in constant time to obtain the  update for each $k$. In particular,
for each such tuple $\tau$, let 
$\psi_\tau = \{ \{ t_1, \ldots, t_j \} \in \mathcal{C}: t_1< \ldots < t_j \}$. (For instance, if $j=2$, $n=4$, 
$\tau=(1,2)$, and $\mathcal{C}$ contains all 3-tuples, then $\psi_\tau=\{ \{1,2,3\}, \{1,2,4\} \}$.)
In each iteration, let $c^1,c^2 \in \psi_\tau$ be the arguments corresponding  to the largest
and second-largest elements in $q(i,\psi_\tau)$;
if $|\psi_\tau|=1$, then define $c^2=\emptyset$; if $|\psi_\tau|=0$, then define both $c^1=c^2=\emptyset$.
(Note that $\emptyset \not \in \overline{\phi}(k)$ for any $k$.)
For any $k$, it must be the case that the number of elements in the set
$\psi_\tau \cap \overline{\phi}^j(k)$ is either $0$ (if any $k\in \{t_1,\ldots, t_{j-1} \}$),
$|\psi_\tau|-1$ (if $k>t_{j-1}$, such that we must ignore exactly one element of $\psi_\tau$ for each $k$), or $|\psi_\tau|$ (if $k \not \in \{t_1,\ldots, t_{j-1}\}$ and $k < t_{j-1}$).
In the first case (intersection is empty), we make no update to $\max_{c_i \in \overline{\phi}^j(k)} q(i,c_i)$.
In the second (intersection is of size $|\psi_\tau|-1$), we update $\max_{c_i \in \overline{\phi}^j(k)} q(i,c_i)$ with
$q(i,c^1)$ if $c^1 \not \ni k$ and with $q(i,c^2)$ otherwise. And in the third (intersection is of size $|\psi_\tau|$), we always update 
$\max_{c_i \in \overline{\phi}^j(k)} q(i,c_i)$ with $q(i,c^1)$. Since $c^1,c^2$ can be computed in time
$|\psi_\tau| \leq n$ and then reused for \emph{each} of the $k$ (in constant-time) for the updates, and since there
are at most $n^{j-1}$ such tuples $\tau$ when scanning the entire $\mathcal{C}$, then this amounts to
a total time of $O(dn^j)$ for each $j$. Summing over all $j=1,\ldots,d$, this yields a running time
of $O(dn^d)$. See Algorithm \ref{alg:findmaxes}. The $\argmax^{1,2}$ function returns the $c^1,c^2$
that give the
largest and second-largest values of the specified function, where $c^2=\emptyset$ if the input set is of size 1,
and $c^1=c^2=\emptyset$ if the input set is empty.

\subsection{Computing Maxes of Sums Except Row $k$}
We can now show how
expressions of the form $b(i',k) = \max_{c_{i'}} \R{i}{k}{c_{i'}}$ and $\overline{b}(i',k) = \max_{c_{i'} \not \ni k} \R{i}{k}{c_{i'}}$ can be computed
efficiently. We first examine the former, which by definition is:
\[
\max_{c_{i'}} \R{i'}{k}{c_{i'}} = \max_{c_{i'}} \left[ S(i', c_{i'}) + \sum_{k' \ne k} \A{i'}{k'}{c_{i'}} \right]
\]
In other words, we need to find the $c_{i'}$ that maximizes $S(i', c_{i'})$ plus
the sum (except the $k$th term) of the $\A{i}{k'}{c_{i'}}$ (see Figure \ref{fig:maxSumE}).
As mentioned above, for each $i',k$, function $\A{i'}{k}{\cdot}$ has only 2 degrees of freedom:
one for $c_{i'} \in \phi(k)$ (the blue regions in Figure \ref{fig:maxSumE})
and one for $c_{i'} \in \overline{\phi}(k)$ (the clear regions); hence, there exist numbers $u,v$
such that $\A{i'}{k}{\phi(k)} = u$ and $\A{i'}{k}{\overline{\phi}(k)} = v$. Assume we have already computed
$q(i,c_{i'}) = \sum_{k'} \A{i'}{k'}{c_{i'}}\ \forall c_{i'}$ (this is the sum over \emph{all} $k$) and also, for
each $k$, the values
$r(k) = \max_{c_{i'} \in \phi(k)} \sum_{k'} \A{i'}{k'}{c_{i'}}$ and
$s(k) = \max_{c_{i'} \in \overline{\phi}(k)} \sum_{k'} \A{i'}{k'}{c_{i'}}$. Then, for any $k$, we can find, in $O(1)$ time, 
$\max_{c_{i'}} \sum_{k' \ne k} \A{i'}{k'}{c_{i'}}$ by ``adjusting'' $\max_{c_{i'}} \sum_{k'} \A{i'}{k'}{c_{i'}}$ as follows:
\[
\max_{c_{i'}} \sum_{k' \ne k} \A{i'}{k'}{c_{i'}} = \max(r(k)-u, s(k)-v)
\]
The latter case ($\max_{c_{i'} \not \ni k} \R{i'}{k}{c_{i'}}$) is even easier since we ignore all $c_{i'} \in \phi(k)$ entirely:
\[
\max_{c_{i'} \not \ni k} \sum_{k' \ne k} \A{i'}{k'}{c_{i'}} = s(k) - v
\]
We have already defined $u=a(i',k)$ and $v=\overline{a}(i',k)$; hence, we have:
\[
b(i',k) = \max(r(k) - a(i',k), s(k) - \overline{a}(i',k)), \qquad \qquad
\overline{b}(i',k) = s(k) - \overline{a}(i',k)
\]

\begin{figure}
\begin{center}
\includegraphics[width=2.25in]{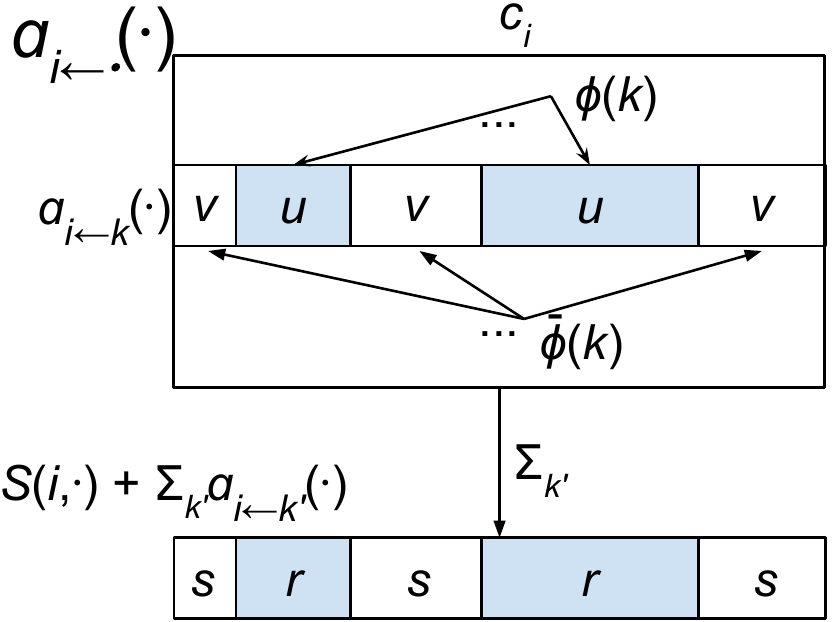}
\end{center}
\caption{For each $i,k$, to compute the max (over $c_{i'}$) of the sum of all rows $k'\ne k$, we can (1) compute the max of
the sum of \emph{all} rows within region $\phi(k)$ and (separately) within region $\overline{\phi}(k)$; (2) adjust each maximum by subtracting the value of row $k$ in region $\phi(k)$ 
and the value of row $k$ in region $\overline{\phi}(k)$, respectively; (3) take the larger result.
}
\label{fig:maxSumE}
\end{figure}

\subsection{Computing $h(k)=\R{k}{k}{k}$}
As the last step, we can compute $h(k)=\R{k}{k}{k} = S(k,\{k\}) + \sum_{k'\ne k} \A{k}{k'}{k} = S(k,\{k\}) + q_k(\{k\}) - a(k,k)$.
This completes the derivation of Algorithm 1.

\subsection{Time Cost Analysis}
As explained in Section \ref{sec:find_all_maxes}, the $\textrm{FindAllMaxes}$ takes time $O(dn^d)$ operations for each $i$.
The function $\textrm{ComputeRhoStats}$ calls $\textrm{FindAllMaxes}$ $n$  times (and also executes $O(n^2)$ further
operations) for a cost of $dn^{d+1}$. The function $\textrm{ComputeAlphaStats}$ takes $O(n^2)$ for the nested for-loops,
and (as explained in Section \ref{sec:q}) a further $O(dn^d)$ for the computation of each $q(i,\cdot)$, amounting to 
$O(dn^{d+1})$ in total.

This completes the proof.



\section{Brute-Force Reassignment}
\label{sec:brute_force_reassignment}
Here is how a brute-force reassignment could work:
We first obtain a set of $k$ singleton clusters with
associated exemplar indices $\mathcal{E} \subset [n]$.
Then we iterate over every possible subset $\mathcal{\tilde E} \subseteq \mathcal{E}$; these represent the compositional
clusters. 
For each $\mathcal{\tilde E}$, we conduct an inner-loop
to iterate over every possible 1-to-1 map from $\mathcal{\tilde E}$ to the set of compositions of $\mathcal{E} \setminus \mathcal{\tilde E}$;
these represent the singleton clusters.
If we consider compositions of at most $d$ exemplars, then we have
$\sum_{i=0}^k C(k,i) P\left(\sum_{d'=2}^d C(k-i, d'), i\right)$
total possible maps,
where $C(k,i)$ and $P(k,i)$ are the numbers of combinations and permutations of $i$ objects from a set of $k$,
respectively. The $P$ arises due to iterating over all 1-to-1 maps.
Note that the number of possible maps grows factorially with $k$, and hence it quickly becomes intractable as
$k$ grows (e.g., for $k=15$ and $d=2$, the number of possibilities is $107770296705436$).



\section{LibriSpeech}
LibriSpeech contains 1000+ hours of recorded English-language speech of people reading audiobooks.
While the dataset contains speech from only individual speakers, we can synthesize speech by
adding  the waveforms of multiple speakers.
Figure \ref{fig:librispeech_examples} shows of an example of how simultaneous speech data is synthesized from LibriSpeech data.

{\bf Compositional embedding model}:
Speaker embeddings were extracted from mel-frequency cepstrum coefficient (MFCC) features (32 coefficients, 0.025s window size, 0.01s step size) using an embedding function $f^\textrm{emb}$ that contains a 2-layer LSTM with 256 hidden units. Composition function $g$ is defined as $g(x_a, x_b) = W_1 x_a + W_1 x_b + W_2 (x_a \odot x_b)$, where $W_1,W_2$ are learnable weights and $x_a, x_b$ are speaker embeddings. $f^\textrm{emb}$ and $g$ were optimized jointly. During training, 15 audio samples from 5 unique speakers (5 labeled with 1 speaker and 10 with 2 speakers) are used to extract reference speaker embeddings using $f^\textrm{emb}$. 20 query speaker embeddings were extracted from the same 5 speakers using $f^\textrm{emb} \& g$, with audio or audio pairs. The distances between reference embeddings and query embeddings are computed and the model is optimized using triplet loss so that the distance between a reference-query pair share the same label is smaller then that of other pairs. After training, the model achieves overall accuracy of 86.9\% on a validation set where each episode contains 20 queries as above.

After function $f^\textrm{emb}$ and $g$ are trained, we selected hyperparameters  based on a separate validation set and then tested on test set. Both the validation set and test set contain 10 groups of data, and all clusters have the same number of samples in each group of data. (For example, in the setting of $l=3,n=120$, there are 6 clusters with labels $\{1\},\{2\},\{3\},\{1,2\},\{1,3\},\{2,3\}$ and each one has 20 samples.) For all methods, hyperparameters are selected for $l=3$ and for $l=5$ (both n=150 and n=495) separately. For CAP/CAP$\subset$ and AP, there is only one hyperparameter, $\gamma$, which we varied over the set $\{1,2,\ldots,7\}$. For AC, there is a distance threshold hyperparameter, which we varied over $\{1,2,3,4,4\}$. These sets of values were chosen in pilot experimentation to give a fair chance to each algorithm; in particular, they were chosen so that the best result, during the validation process, did not fall on the boundary of these sets. 
During the message-passing process, we dampened the values returned by ComputeAlphaStats and ComputeRhoStats
using a damping value of $\lambda = 0.65$: $\textrm{Val} = \textrm{OldVal} * \lambda + \textrm{NewVal} * (1 - \lambda)$.
This value for $\lambda$ was used for CAP, CAP$\subset$, and AP.

\section{OmniGlot}
OmniGlot contains images of handwritten symbols from a variety of languages.
Figure \ref{fig:omniglot_examples} shows 4 groups of examples used in the experiment. Each group contains images with cluster labels $\{1\},\{2\},\{3\},\{1,2\},\{1,3\},\{2,3\}$.
{\bf Compositional embedding model}: For the image embedding function $f^\textrm{emb}$ we used ResNet18. Composition function  $g$ is defined the same as for LibriSpeech. The training procedure of $f^\textrm{emb}$ and $g$ are the same as for LibriSpeech. After training, the embedding model achieves overall accuracy of 75.0\% on validation set.
After function $f^\textrm{emb}$ and $g$ are trained, the hyperparameters are selected in the same way, and from the same sets, as in the LibriSpeech experiment. We used damping just like for LibriSpeech.


\end{document}